\documentclass{article}

\usepackage{arxiv}

\usepackage[utf8]{inputenc} 
\usepackage[T1]{fontenc}    
\usepackage{hyperref}       
\usepackage{url}            
\usepackage{booktabs}       
\usepackage{amsfonts}       
\usepackage{nicefrac}       
\usepackage{microtype}      
\usepackage{graphicx}
\usepackage{natbib}
\usepackage{doi}

\usepackage{amsmath}
\usepackage{amsthm}
\usepackage{algorithm}
\usepackage{algorithmic}
\usepackage[switch]{lineno}
\usepackage[table,usenames,dvipsnames]{xcolor}
\usepackage[small]{caption}
\usepackage{subcaption}
\usepackage{amssymb}
\usepackage{mathtools}
\usepackage{bm}
\usepackage{xspace}
\usepackage{enumitem}
\usepackage{url}
\usepackage{comment}

\title{An Adaptive Method Stabilizing Activations for Enhanced Generalization}

\author{Hyunseok Seung \\
	Department of Statistics\\
	University of Georgia\\
	Athens, GA \\
	\texttt{hseung@uga.edu} \\
	\And
	Jaewoo Lee \\
	School of Computing\\
	University of Georgia\\
	Athens, GA \\
	\texttt{jaewoo.lee@uga.edu} \\
	\And
	Hyunsuk Ko \\
	School of Electrical Engineering\\
	Hanyang University\\
	Ansan, South Korea \\
	\texttt{hyunsuk@hanyang.ac.kr} \\
}

\date{}

\renewcommand{\headeright}{}
\renewcommand{\undertitle}{}
\renewcommand{\shorttitle}{An Adaptive Method Stabilizing Activations for Enhanced Generalization}

\hypersetup{
  pdftitle={An Adaptive Method Stabilizing Activations for Enhanced Generalization},
  pdfsubject={cs.LG},
  pdfauthor={Hyunseok Seung},
  pdfkeywords={Deep learning optimization, Adaptive gradient methods, Gradient preconditioning}
}

\let\en=\ensuremath

\DeclareMathOperator{\E}{\mathbb{E}}          

\DeclarePairedDelimiter{\norm}{\lVert}{\rVert}
\DeclarePairedDelimiter{\abs}{\lvert}{\rvert}
\DeclarePairedDelimiter{\ip}{\langle}{\rangle}

\DeclareMathOperator{\diag}{diag}      

\DeclareMathOperator{\vect}{vec}                

\renewcommand{\vec}[1]{\en{\bm{\mathrm{#1}}}}

\newcommand{\hvec}[1]{\en{\hat{\vec{#1}}}}      
\newcommand{\mat}[1]{\en{{\bm{\mathrm{#1}}}}}
\newcommand{\grad}[0]{\en{\nabla}}

\newcommand{\R}[0]{\mathbb{R}}

\newcommand{\pdv}[2]{\frac{\partial #1}{\partial #2}}

\renewcommand{\th}[0]{\textsuperscript{th}\xspace}

\newcommand{\red}[1]{\textcolor{OrangeRed}{#1}}
\newcommand{\green}[1]{\textcolor{PineGreen}{#1}}
\newcommand{\blue}[1]{\textcolor{blue}{#1}}
\newcommand{\yellow}[1]{\textcolor{YellowOrange}{#1}}

\newcommand{\adaact}[0]{\textsc{AdaAct}\xspace}

\theoremstyle{plain}
\newtheorem{theorem}{Theorem}[section]

\newtheorem{lemma}[theorem]{Lemma}
\newtheorem{corollary}[theorem]{Corollary}
\theoremstyle{definition}
\newtheorem{definition}[theorem]{Definition}
\newtheorem{assumption}[theorem]{Assumption}
\theoremstyle{remark}

\begin{document}
\maketitle

\begin{abstract}
	We introduce \adaact, a novel optimization algorithm that adjusts
learning rates according to activation variance. Our method enhances
the stability of neuron outputs by incorporating neuron-wise
adaptivity during the training process, which subsequently leads to
better generalization---a complementary approach to conventional
activation regularization methods. Experimental results demonstrate
\adaact's competitive performance across standard image classification benchmarks.
We evaluate \adaact on CIFAR and ImageNet, comparing it with other state-of-the-art methods.  
Importantly, \adaact effectively bridges the gap between the convergence speed of
Adam and the strong generalization capabilities of SGD, all while
maintaining competitive execution times. Code is available at \href{https://github.com/hseung88/adaact}{https://github.com/hseung88/adaact}
\end{abstract}


\vspace{1em}
This is the extended version of the paper published in the \textit{2024 IEEE International Conference on Data Mining Workshops (ICDMW)}, \textcopyright\ IEEE. The published version is available at: \href{https://doi.org/10.1109/ICDMW65004.2024.00007}{10.1109/ICDMW65004.2024.00007}

\section{Introduction}
\label{sec:introduction}
Adaptive gradient methods such as Adam~\cite{kingma2015adam} and its
variants~\cite{liu2020Adam} have been the method of
choice for training deep neural networks (NNs) 
due to their faster convergence compared to
SGD~\cite{Sutskever2013OnTI}. However, a line of
studies~\cite{Wilson2017TheMV,Chen2018ClosingTG,Reddi2019OnTC} has
reported the cases in which these adaptive
methods diverge or result in worse generalization performance than SGD. While
several optimizers such as SWAT~\cite{keskar2017improving},
AdaBound~\cite{Luo2019AdaBound}, and Padam~\cite{chen2020closing}
have been proposed to mitigate the issue, these methods mostly focus
on establishing optimization bounds on the training objective, 
ignoring the generalization and stability properties of the model being trained. 

Recent work has investigated the connection between activation
stability and generalization properties of neural networks and
empirically demonstrated that stabilizing the output 
can help improve the generalization performance.
These works proposed approaches to maintain stable output distribution
among layers, which includes explicitly normalizing the
activations~\cite{srivastava14dropout,ioffe15batch,Ba2016LayerN}, 
adding a loss term to penalize the activation
variance~\cite{krueger2015regularizing,Littwin2018RegularizingBT,Ding2019RegularizingAD},
or regularizing the output into the standard normal
distribution~\cite{Joo2020RegularizingAI}.
Orthogonal to prior approaches that rely on activation
regularization, in this work,
we devise an optimization method, called \adaact, that directly promotes stable
neuron outputs during training.
Specifically, to stabilize the activations during training, \adaact
carefully controls the magnitude of updates according 
to the estimated \emph{activation variance}.
This is in contrast to vast majority of other adaptive gradient methods 
that adapt to gradient variance. Our strategy involves taking 
smaller steps when encountering high activation variance and, conversely, taking
larger steps in the presence of low activation variance. This is achieved by maintaining the running mean of activation
variance and scaling the gradient update inversely proportional to the
square root of the variance.
Seemingly our method may look similar to
FOOF~\cite{Benzing2022GradientDO} or
LocoProp~\cite{Amid2021LocoPropEB} as these methods use activation
covariance matrix to precondition the gradient.
However, we emphasize that our method is developed with a completely different
motivation of activation stabilization via variance adaptation,
while their analyses primarily focus on investigating the
effectiveness of Kronecker-Factored approximation in KFAC~\cite{martens2015optimizing}
and the connection between their optimizer and second-order methods.
In addition, these methods are inefficient as they require storing 
a large covariance matrix for each layer and involve costly matrix
inversion operation.
In contrast, our method assumes the independence between activations
and only computes the variance of individual activations (which
corresponds to the diagonal entries in the covariance matrix).
Our method is also different from other adaptive gradient methods that
maintain a per-parameter learning rate in the sense that it applies a
less aggressive adaptation strategy to avoid the pitfall of too much
adaptation. In our method, the parameters that
interact with the same input share the same learning rate.

To demonstrate the effect of adapting gradient to activation
variance, we train
LeNet-5~\cite{lecun98lenet5} 
on CIFAR10~\cite{krizhevsky2009learning} dataset using our proposed
method and visualize the activation variance (averaged over entire
training iterations) and test
accuracy in Figure~\ref{fig:cifar10_lenet5}. To calculate the averaged activation variance, 
we first flatten the activations of these layers and compute the
variance for each activation. Then we average these variances over
iterations.
As shown in the figure, the network trained using \adaact yields the
smallest activation variance in all layers and achieves higher test
accuracy at the end of training compared to momentum SGD and
Adam. Adam shows faster convergence at the early stage
of training thanks to its fast adaptation capability, which results in higher activation variance. See Figure~\ref{fig:cifar10_lenet5_actvar_all} in Appendix~\ref{apdx:adaact_actvar_bound} for the unaveraged activation variance plots.

\begin{figure}[tb]
\centering
\includegraphics[width=0.6\textwidth]{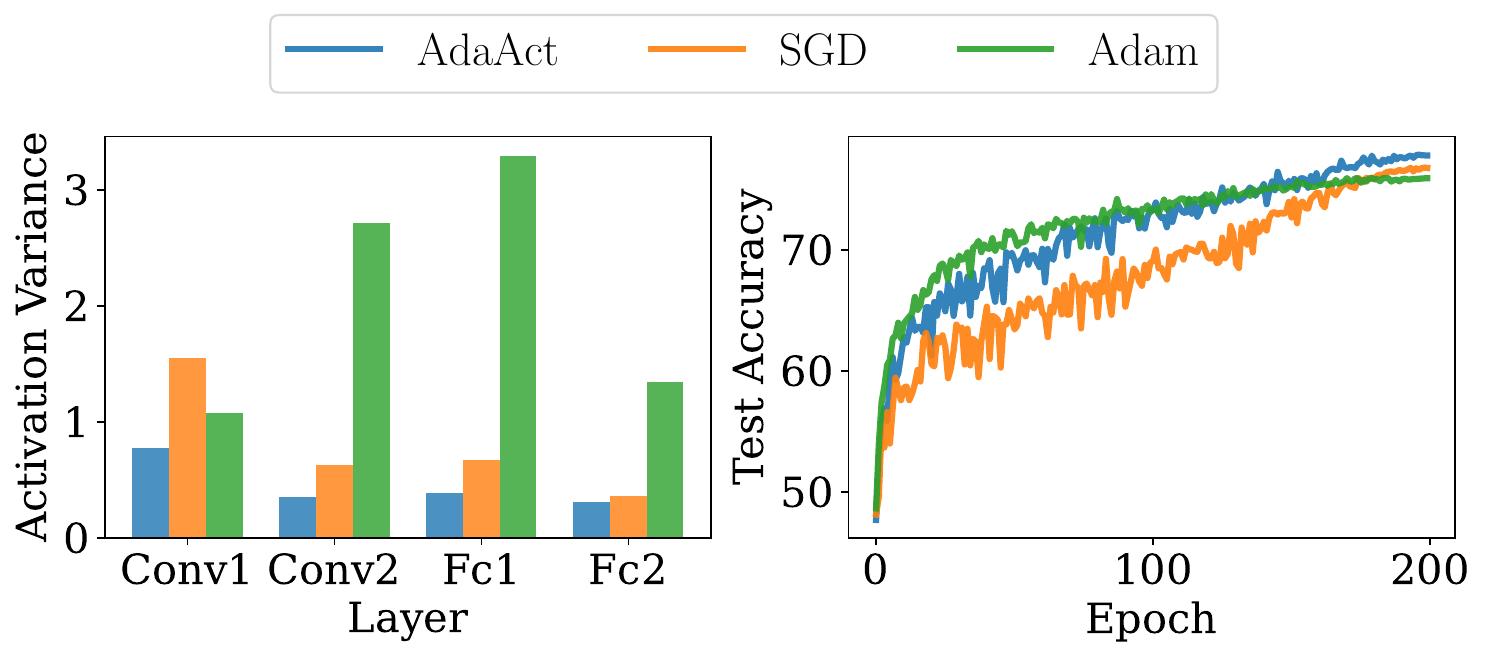}
\caption{(Left) Averaged activation variance from each hidden layer and (Right) test accuracy training LeNet-5 on CIFAR10.}
\label{fig:cifar10_lenet5}
\end{figure}

The key contributions of our work can be summarized as follows:
\begin{itemize}[leftmargin=*]
\item We propose a novel optimizer that stabilizes the neuron outputs via
  activation variance adaptation.
\item Our proposed method demonstrates improved generalization compared to
state-of-the-art adaptive methods.
Its convergence speed is similar to that of Adam while at the end of
training it achieves good generalization performance comparable to that
of highly tuned SGD.
\item To evaluate the performance of proposed method, we conduct extensive
experiments on image classification task with CIFAR 10/100 and ImageNet
dataset using various architectures, including ResNet, DenseNet, and
ViT.
Importantly, it 
achieves enhanced performance while maintaining a comparable execution
time to other adaptive methods.
\end{itemize}

\section{Related Work}
\label{sec:related}
In this section, we provide an overview of relevant literature that
both underpins and complements our work.  

\paragraph{Adaptive Methods.} Adaptive methods such as
AdaGrad~\cite{duchi2011adaptive}, RMSProp~\cite{hinton12rmsprop}, and
Adam~\cite{kingma2015adam} have enhanced NN training due to their
superior convergence speeds compared to
SGD~\cite{Sutskever2013OnTI}. However, concerns have emerged about
over-specialization with these methods, potentially impacting model
generalization. Specifically, \cite{Wilson2017TheMV} pointed out that
these methods might accentuate the generalization gap compared to
SGD. Additionally, \cite{Chen2018ClosingTG} highlighted high
adaptivity as a root cause, and \cite{Reddi2019OnTC} mentioned
contexts where Adam may not 
converge. 
In response to these challenges, a variety of optimization methods
have been proposed. Nadam~\cite{dozat2016incorporating} synergizes the
advantages of Adam and Nesterov's accelerated gradient
to promote better convergence and generalization.
Padam~\cite{Chen2018ClosingTG}
features a tunable hyperparameter to bridge the gap between Adam and
SGD. AdamW~\cite{Loshchilov2019DecoupledWD} decouples weight decay
from adaptive learning rates.
AdaBound~\cite{Luo2019AdaBound} modulates the learning
rates in adaptive methods, bounding them based on the traditional SGD
approach. AMSGrad~\cite{Reddi2019OnTC} enforces bounds on learning
rates by leveraging the maximum observed moving average value.
AdaBelief~\cite{zhuang2020adabelief} monitors the moving average
of squared gradient discrepancies versus their respective moving
average, differentiating genuine gradient noise from actual gradient
shifts. Adai~\cite{Xie2020AdaptiveID} isolates the influences of the
adaptive learning rate and momentum within Adam dynamics.
Lastly, Radam~\cite{Liu2020OnTV} incorporates a term that
tempers the adaptivity of learning rates in initial training phases,
fostering more consistent and dependable training. 

\paragraph{Activation Regularization.} Recent studies have emphasized
the significance of regularizing activations for better model
generalization. Dropout~\cite{srivastava14dropout} randomly nullifies
activations, preventing overreliance. Batch
normalization~\cite{ioffe15batch} and layer
normalization~\cite{Ba2016LayerN} maintain consistent activation
distributions across batches or features
respectively. \cite{krueger2015regularizing} proposed a regularization
technique for recurrent neural networks (RNNs) that mitigates abrupt
activation changes, while \cite{Merity2017RevisitingAR} delved deeper
into activation regularization for language tasks with
RNNs. \cite{Littwin2018RegularizingBT} and
\cite{Joo2020RegularizingAI} targeted consistent activations across
batches; the former used the variance of their sample-variances, while
the latter employed the Wasserstein
distance. \cite{Ding2019RegularizingAD} suggested distribution loss
for binarized networks, and \cite{Fu2022NeuronWS} argued consistent
neuronal responses enhance generalization in NNs. 

\paragraph{Covariance-based Gradient Preconditioning.}
\cite{Ida2016AdaptiveLR} introduced an adaptive method that
preconditions gradient descent using the gradient covariance matrix,
different from our approach of using the activation covariance matrix
as a preconditioner.
FOOF~\cite{Benzing2022GradientDO} explicitly utilizes
activation covariance for gradient preconditioning. In parallel,
LocoProp~\cite{Amid2021LocoPropEB} introduced a framework of layerwise
loss construction, and their update equation aligns with FOOF's when
employing a local squared loss. Eva~\cite{Zhang2023EvaPS} proposed a
second-order algorithm that utilizes a variant of the two covariance
matrices from KFAC, leveraging the Sherman-Morrison formula. 


\section{Preliminaries}
\label{sec:prelim}
We consider solving the following optimization problem:
\begin{equation} \label{eq:opt_prob}
  \min_{\vec{\theta} \in \R^d} F(\vec{\theta}) := \E_{\xi\sim
    \mathcal{D}}\left[f(\vec{\theta}; \xi)\right]\,,
\end{equation}
where $f:\R^d \to \R$ is differentiable and possibly
nonconvex in $\vec{\theta}$ and $\xi$ is a random variable following
an unknown but fixed distribution $\mathcal{D}$. In the context of machine
learning, $F$ corresponds to the empirical risk, i.e.,
$F(\vec{\theta}) = \mathcal{L}(\vec{\theta}; S) :=
\frac{1}{N}\sum_{i=1}^N \ell(\vec{\theta}; \vec{x}_i, y_i)$, where $\ell$ is a loss
function, $S = \{(\vec{x}_i, y_i)\}_{i=1}^N$ is the training dataset, and $\vec{\theta}$ corresponds
to model parameters. 

\subsection{Notations}
For vectors, we use element-wise operations unless specified
otherwise. $(\vec{x})_{i}$ denotes the $i$-th coordinate of
$\vec{x}$. $\norm{\vec{x}}$ represents 
$L_2$ norm unless stated otherwise. We use $[N]$ to denote the set
$\{1,2, \dots, N\}$, $\otimes$ to represent the Kronecker product, and $\odot$ for the Hadamard product. 
Consider a feed-forward NN consisting of $L$ layers trained on a
dataset $S = \{(\vec{x}_i, y_i)\}_{i=1}^n$.
Let $\mat{W}$ and $\vec{b}$ 
be the weight and bias of layer $\ell\in [L]$. It is often convenient
to include the bias term into the weight $\mat{W}$ as the last column:
$\mat{\Theta}^{(\ell)} = \begin{bmatrix} \mat{W}^{(\ell)} & 
    \vec{b}^{(\ell)}\end{bmatrix} \in \R^{m_{\ell}\times
    (m_{\ell-1}+1)}\,$.
We augment each $\vec{a}^{(\ell)}\in \R^{m_{\ell}}$ by adding a $1$ to
its last entry and denote it by $\tilde{\vec{a}}^{(\ell)}\in \R^{m_\ell+1}$.
The forward step of our NN is given by
\begin{align*}
  \vec{z}^{(\ell)}
  &= \mat{\Theta}^{(\ell)} \tilde{\vec{a}}^{(\ell-1)}\in \R^{m_{\ell}}\,,
  \qquad  \vec{a}^{(\ell)}=
    \phi(\vec{z}^{(\ell)}) 
    \in \R^{m_{\ell}}\,,\\
  \vec{\theta}^{(\ell)}
  &= \vect(\mat{\Theta}^{(\ell)}) \in \R^{m_{\ell}(m_{\ell-1}+ 1)}\,,
\end{align*}
where $\vec{z}$, $\vec{a}$, and $\phi$ represent the pre-activations,
activations, and an activation function, respectively, and
$\tilde{\vec{a}}^{(0)} = \vec{x}$.
The vectorization operator, denoted by $\vect(\cdot)$,
takes 
$\mat{X}\in \R^{m\times n}$ as input and returns a vector
$\vect(\mat{X})\in \R^{mn}$ of length
$mn$. That is, $\vect(\mat{X}) =
  \begin{bmatrix}
    \mat{X}_{*,1}^{\intercal} & \mat{X}_{*,2}^{\intercal} & \cdots & \mat{X}_{*,n}^{\intercal}
  \end{bmatrix}^{\intercal}\,$, where $\mat{X}_{*,j}$ denotes the
  $j$\th column of matrix $\mat{X}$.

\subsection{Kronecker Factored Approximate Curvature}
\cite{martens2015optimizing} introduced KFAC which approximates the Fisher information matrix (FIM) as 
$(\widetilde{\mat{F}})_{i,j} = \mat{A}_{i-1, j-1}\otimes
\mat{P}_{i,j}$, where $\mat{A}_{i,j}=\E\left[\tilde{\vec{a}}^{(i)}(\tilde{\vec{a}}^{(j)})^\intercal\right]$
denotes the covariance of the activations from layer $i$ and $j$, and
$\mat{P}_{i,j}=\E\left[\pdv{\mathcal{L}}{\vec{z}^{(i)}} 
  \pdv{\mathcal{L}}{\vec{z}^{(j)}}^\intercal\right]$ represents the
covariance of pre-activation gradients between layer $i$ and  $j$. 
Assuming the independence between layer $i$ and $j$ for $i\neq j$,
KFAC only computes the diagonal blocks of FIM, denoted by
$\mat{A}^{(\ell-1)}\otimes \mat{P}^{(\ell)} = \mat{A}_{\ell-1,\ell-1} \otimes \mat{P}_{\ell,\ell}$
, which results in the following update
rule for layer $\ell$ at iteration $t$.
\begin{align}
  \label{eq:kfac}
  \vec{\theta}_{t+1}^{(\ell)}
  &=\vec{\theta}_{t}^{(\ell)} -\eta (\mat{A}_t^{(\ell-1)} \otimes \mat{P}_t^{(\ell)})^{-1}
    \vect(\nabla_{\mat{\Theta}^{(\ell)}}{\mathcal{L}_t}) \nonumber \\ 
    &=\vec{\theta}_{t}^{(\ell)}-\eta
    \vect{\big((\mat{P}_{t}^{(\ell)})^{-1}\grad_{\mat{\Theta}^{(\ell)}}{\mathcal{L}_{t}}(\mat{A}_{t}^{(\ell-1)})^{-1}\big)}, 
\end{align}
where $\eta$ is learning rate and
$\nabla_{\mat{\Theta}^{(\ell)}}{\mathcal{L}_t} \in \R^{m_{\ell}\times
  (m_{\ell-1}+1)}$ is the gradient of $\mathcal{L}$ w.r.t. the
parameters of layer $\ell$ evaluated at time $t$.
\cite{Benzing2022GradientDO} argued that the pre-activation gradient
term $\mat{P}$, in fact,
does not contribute to superior performance of KFAC
and proposed the following update:
\begin{equation}
  \label{eq:foof}
  \vec{\theta}_{t+1}^{(\ell)}=\vec{\theta}_{t}^{(\ell)}-\eta
  \vect{\big(\grad_{\mat{\Theta}_{t}}{\mathcal{L}}(\mat{A}_{t}^{(\ell-1)})^{-1}\big)}\,.
\end{equation}
The above equation is derived by applying the principle that
an update of the weight matrix explicitly changes the layer's outputs
(pre-activations) into their gradient direction (pre-activation
gradients). 
Mathematically, this can be 
expressed as $(\mat{\mat{\Theta}} + \Delta\mat{\mat{\Theta}})\tilde{\vec{a}}
= \vec{z} + \eta \frac{\partial \mathcal{L}}{\partial \vec{z}}$ and
 such $\Delta\mat{\mat{\Theta}}$ is obtained by 
solving
$\min_{\Delta\mat{\mat{\Theta}}}{||\left(\Delta\mat{\mat{\Theta}}\right)\tilde{\vec{a}}-\eta
  \frac{\partial \mathcal{L}}{\partial
    \vec{z}}||^{2}}+\frac{\lambda}{2}||\Delta\mat{\mat{\Theta}}||^{2}$.
This suggests that obtaining optimized neuron outputs in
NNs is closely connected to preconditioning gradients with
activation covariance,
which motivated the activation variance-based adaptation in~\adaact.


\section{Algorithm}
\label{sec:algo}
In this section, we introduce \adaact, for
solving the optimization problem~\eqref{eq:opt_prob}. 
The pseudocode of algorithm is presented in
Algorithm~\ref{alg:adaact}.

For layer $\ell$, the input activation covariance matrix
$\mat{A}_t^{(\ell-1)}$ can be estimated using the samples in minibatch $\mathcal{B}_t$.
\begin{align}
  \mat{A}_t^{(\ell-1)}
  &=\E\left[\tilde{\vec{a}}^{(\ell-1)}(\tilde{\vec{a}}^{(\ell-1)})^{\intercal}\right]
  \in \R^{(m_{\ell-1}+1)\times (m_{\ell-1}+1)} \nonumber \\
  &\approx \frac{1}{\abs{\mathcal{B}_t}}\sum_{i\in \mathcal{B}_t}
    \tilde{\vec{a}}^{(\ell-1)}_i(\tilde{\vec{a}}^{(\ell-1)}_i)^{\intercal} \,,
    \label{eq:act_cov}
\end{align}
where $\vec{a}^{(\ell-1)}_i$ denotes the activation of layer $\ell-1$
when the input to the network is the $i$\th example $\vec{x}_i$ in the training
set.
The covariance matrix in~\eqref{eq:act_cov} 
could be large for many modern large scale neural networks (e.g., ViT). For
a network with $L$ layers, it requires storing $\sum_{\ell=0}^{L-1}
m_{\ell}^2$ entries. Even worse, computing its inverse takes
$\mathcal{O}(m_{\ell}^3)$ time in general.

\begin{figure}[tb]
\centering
\includegraphics[width=0.8\textwidth]{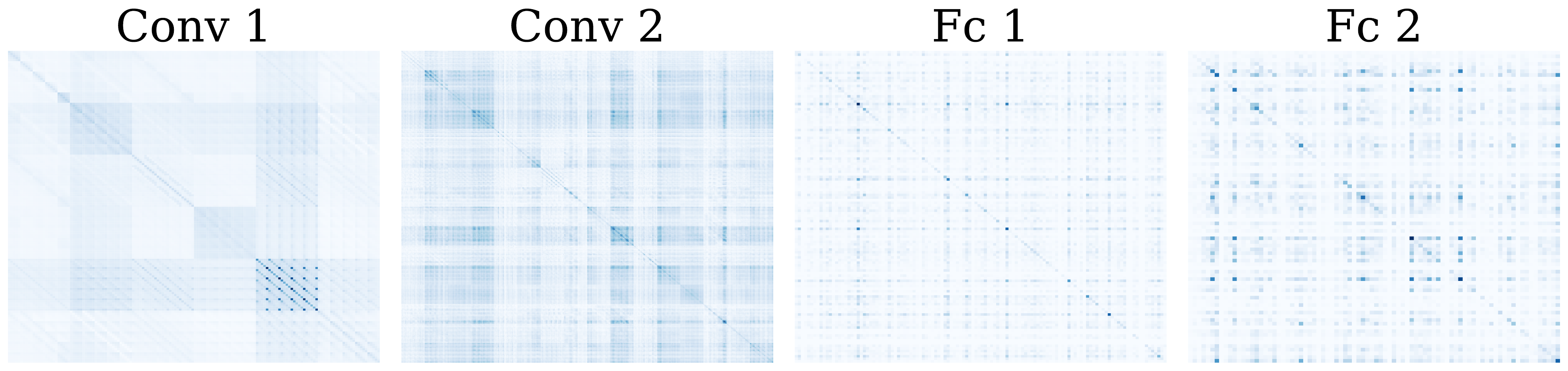}
\caption{Heatmap of the activation covariance from the hidden layers of LeNet-5 trained on CIFAR10.}
\label{fig:heatmap}
\end{figure}

Figure~\ref{fig:heatmap} presents heatmaps of the activation
covariance of each hidden layer.
Due to the use of ReLU activation function in many modern neural networks,
the activation covariance matrix is sparse, and the entries in
diagonal positions tend to have relatively larger magnitude than other
entries. 
From these observations, \adaact approximates $\mat{A}_t^{(\ell-1)}$
as a diagonal matrix --- this results in lower space complexity than
Adam -- and applies the weighted averaging.
Line~\ref{alg:v_update} computes the exponential moving average
(EMA) of the second moment of activations where both $\mat{V}_t$ and
$\widetilde{\mat{A}}_t$ belong to $\R^{(m_{\ell-1}+1)\times
(m_{\ell-1}+1)}$. While the algorithm appears to resemble Adam, it was
derived from a different perspective. Specifically, in
Line~\ref{alg:v_update} of Algorithm~\ref{alg:adaact}, \adaact
replaces the EMA of squared gradient in Adam with that of activation
variance.  
Our algorithm can be viewed as dynamically adjusting the learning
rates according to the variance of activations.

\begin{algorithm}[tb]
\caption{\adaact}
\label{alg:adaact}
\textbf{Require}: Learning rate $\eta_{t}$, Momentum $\beta_{1}=0.9,\ \beta_{2}=0.999$, Weight decay $\lambda$, Numerical stability $\epsilon$\\
\textbf{Initialize}: $\vec{\theta}_{0}$, $\mat{M}_{0} = \mat{O}$, $\mat{V}_{0} = \mat{O}$ \\
\textbf{Output}: $\vec{\theta}_{t}\in \R^{m_{\ell}(m_{\ell-1}+1)}$
\begin{algorithmic}[1] 
\FOR{$t$=1,2,3, \dots}
\STATE Draw a mini-batch $\mathcal{B}_t$ of samples.
\STATE $\circ$ EMA of activation variance:
\STATE $\widetilde{\mat{A}}_{t} = \frac{1}{|\mathcal{B}_t|}\sum_{i\in \mathcal{B}_t}{\diag(\tilde{\vec{a}}_{i}\tilde{\vec{a}}_{i}^\intercal)} \in \R^{(m_{\ell-1}+1)\times (m_{\ell-1}+1)}$
\STATE $\mat{V}_{t} = \beta_{2}\mat{V}_{t-1} +
(1-\beta_{2})\widetilde{\mat{A}}_{t}$ \label{alg:v_update}
\STATE $\widehat{\mat{V}}_{t} = \mat{V}_{t} \ / \ (1-\beta_{2}^{t})$

\STATE $\circ$ EMA of gradient:
\STATE $\mat{G}_{t} = \frac{1}{|\mathcal{B}_t|}\sum_{i\in \mathcal{B}_t}\grad{\ell}(\vec{\Theta}_t; \vec{x}_i) \in \R^{m_{\ell}\times (m_{\ell-1}+1)}$
\STATE $\mat{M}_{t} = \beta_{1}\mat{M}_{t-1} + (1-\beta_{1})\mat{G}_{t}$ \label{alg:grad_momentum}
\STATE $\widehat{\mat{M}}_{t} = \mat{M}_{t} \ / \ (1-\beta_{1}^{t})$

\STATE $\circ$ Variance adaptation:
\STATE $\widehat{\mat{G}}_{t} =
\widehat{\mat{M}}_{t}\left(\sqrt{\widehat{\mat{V}}_{t}} +
  \epsilon\mat{I} \right)^{-1}$ \label{alg:var_adapt}
\STATE $\widehat{\vec{g}}_{t} = \vect{(\widehat{\mat{G}}_{t})} \in \R^{m_{\ell}(m_{\ell-1}+1)}$
\STATE $\circ$ Update with decoupled weight decay:
\STATE $\vec{\theta}_{t} =\vec{\theta}_{t-1} - \eta_{t}\left(
  \widehat{\vec{g}}_{t} + \lambda \vec{\theta}_{t-1} \right)$ \label{alg:theta_update}
\ENDFOR
\end{algorithmic}
\end{algorithm}

Two important remarks are in order. First, existing adaptive gradient
methods maintain and adjust the learning rates
\emph{parameter-wise} while
the standard SGD uses a single global learning rate. \adaact takes a
middle ground between these two schemes and adjusts the 
learning rates \emph{neuron-wise}. 
In other words, the parameters that receive the same input features
share the same learning rate. 
While the use of parameter-wise learning rates has shown to be effective in
achieving faster convergence, it is often postulated as the main
culprit of poor generalization performance of adaptive gradient
algorithms~\cite{Wilson2017TheMV}.
Second, the FOOF algorithm also makes use of activation
covariance matrix. However, it is mainly motivated by the fact that
the activation term in KFAC, $\mat{A}_t$ in~\eqref{eq:kfac}, is
sufficient to obtain good performance, and it does not attempt to perform
variance adaptation. We empirically observed that scaling the learning
rate inversely proportional to the \emph{square root} of activation variance
is important, and removing the square root results in degraded
performance. The key features of our algorithm are described below.

\paragraph{Scaled Activation Variance.}
\adaact divides an update by
$\widehat{\mat{V}}_t^{p}$ with $p=0.5$, i.e., 
the square root of activation variance (see Line~\ref{alg:var_adapt}).
The use of square root was derived in AdaGrad by considering the
optimal step size in hindsight to minimize the regret in online learning.
Through experiments, we observed that 
$p=0.5$ achieves better performance than other values,
even better than when 
the full covariance matrix is used. See Figure~\ref{fig:choose_p}. 
\begin{figure}[tb]
    \centering
    \includegraphics[width=0.6\textwidth]{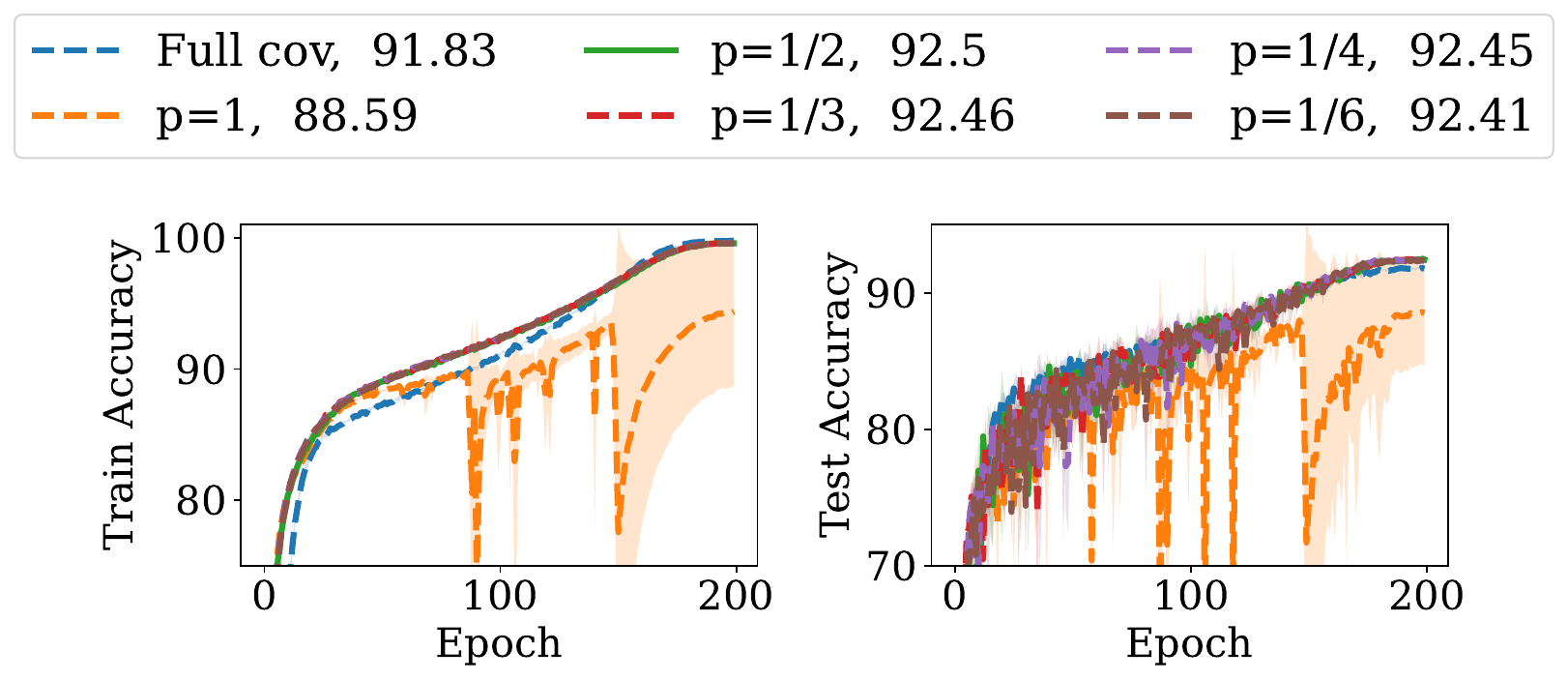}
    \caption{Train and test accuracy of ResNet-20 on CIFAR10 with varying value of $p$. Test accuracy values are indicated in the legends.}
    \label{fig:choose_p}
\end{figure}
The same was also observed in~\cite{Chen2018ClosingTG}.
When the network
exhibits high activation variance, indicating strong responses to
different inputs by individual neurons, 
\adaact uses smaller optimization
steps. Conversely, when the network has low activation variance,
suggesting consistent neuron responses to inputs, it takes larger
optimization steps. This emphasis on stable activations enhances
overall neuron stability during training, fine-tuning the
optimization process to accommodate individual neuron behavior.

\paragraph{Convolutional layer.}
In CNNs, 
activations are 4D tensors of 
shape (batch(B), channel(C), height(H), width(W)). Viewing a
convolution as a matrix-vector product, they are unfolded and
reshaped into a 2D matrix by extracting patches at each spatial
location and flattening into vectors, similar to im2col operation in
GEMM-based implementation of convolution. There are $H\times W$
spatial locations and, for each location, we have a patch flattened
into a vector of size $C\times \kappa\times \kappa$, where $\kappa$ is
the size of kernel. 
This converts the convolution operations into matrix multiplications,
enabling the application of our algorithm initially devised for fully
connected layers to convolutional layers.

\paragraph{Hyperparameters.}
Through a simple grid-based hyperparameter search, we discovered that
our algorithm performs effectively with relatively high learning rates,
typically around 0.1, while many other adaptive methods primarily use
much smaller values, e.g., 0.001.
Regarding weight decay, we adopt the decoupled weight decay~\cite{Loshchilov2019DecoupledWD} and recommend using 
the values smaller than the default 0.01 in AdamW. 
We observed
that employing a higher weight decay value makes our algorithm
converge similarly to SGD, achieving comparable generalization with
it. Conversely, using a lower weight decay value enables fast
convergence similar with  
Adam and its variants while still maintaining
improved generalization.


\section{Analysis of AdaAct}
\label{sec:convergence}
In this section, 
we analyze the convergence and generalization properties of \adaact.
For illustration purpose, we consider feed-forward networks consisting
of linear layers, but our results can also be generalized to other
types of layers.

\subsection{Convergence Analysis}
The convergence guarantee of \adaact can be established using the
framework due to~\cite{chen2018on}. 
For self-completeness, we provide a proof for the case in which the momentum factor $\beta_1$ is fixed i.e., $\beta_{1,t}=\beta_1$, for $\forall t$, in Appendix~\ref{proof:thm3.1}.
We make the following standard assumptions in stochastic optimization.

\newtheoremstyle{named}{}{}{\itshape}{}{\bfseries}{}{.5em}{\thmnote{#3 }}

\theoremstyle{named}
\newtheorem*{namedassumption}{Assumption}
\newtheorem*{namedtheorem}{Theorem}

\begin{enumerate}[wide=1pt,noitemsep,label=\textbf{A\arabic*}.,ref={A\arabic*}]
\item \label{as:A1} $f$ is differentiable and has $L$-Lipschitz gradient,
  i.e. $\forall \vec{x}, \vec{y}$,
  $||\grad{f(\vec{x})}-\grad{f(\vec{y})}||_{2} \leq
  L||\vec{x}-\vec{y}||_{2}$. It is also lower bounded,
  i.e. $f(\vec{x}^{*}) > -\infty$ where $\vec{x}^{*}$ is an optimal
  solution. 
\item \label{as:A2} At time $t$, the algorithm can access a bounded
  noisy gradient and the true gradient is bounded,
  i.e. $||\grad{f(\vec{\theta}_{t})}||_{2} \leq H,\
  ||\vec{g}_{t}||_{2} \leq H,\ \forall t > 1$.
\item \label{as:A3}
The noisy gradient is unbiased and the noise is independent,
i.e. $\vec{g}_{t}=\grad{f(\vec{\theta}_{t})} + \vec{\zeta}_{t},\
\E[\vec{\zeta}_{t}] = 0$ and $\vec{\zeta}_{i}$ is independent of
$\vec{\zeta}_{j}$ if $i\neq j$. 
\end{enumerate}
\begin{theorem}[Theorem 3.1 in~\cite{chen2018on}]
\label{thm:3.1}
Suppose that assumptions \ref{as:A1}-\ref{as:A3} are
satisfied, 
$\beta_{1,t} = \beta_1$ for $\forall t$, and
let $\gamma_t = \min_{j\in [d]}\min_{\{\vec{g}_i\}_{i=1}^t}
  \frac{\eta_t}{\sqrt{\hvec{v}_t}}$. Then we have
  \begin{align}
    \label{ineq:thm3.1}
    &\min_{t\in [T]}\E\left[\norm{\grad{f}(\vec{x}_t)}^2\right] \nonumber\\
    &\leq \Biggl\{
      \E\left[C_1\sum_{t=1}^T \norm*{\frac{\eta_1\vec{g}_t}{\sqrt{\hvec{v}_t}}}^2
      + C_2 \sum_{t=2}^T \norm*{\frac{\eta_t}{\sqrt{\hvec{v}_t}} -
      \frac{\eta_{t-1}}{\sqrt{\hvec{v}_{t-1}}}}_1 \right. \nonumber\\
    &\quad \left. + C_3 \sum_{t=2}^T\norm*{\frac{\eta_t}{\sqrt{\hvec{v}_t}} -
      \frac{\eta_{t-1}}{\sqrt{\hvec{v}_{t-1}}}}^2      
      \right] + C_4 \Biggr\}\Bigg/\sum_{t=1}^T\gamma_t\,,
  \end{align}
where $C_{1}, C_{2}, C_{3}$ are constants independent of $d$ and
$T$, $C_{4}$ is a constant independent of $T$, the expectation is
taken with respect to all the randomness corresponding to
$\{\vec{g}_{t}\}$. 
\end{theorem}
To compute the convergence rate for \adaact, we make the following
additional assumptions. 
%
%
\begin{enumerate}[wide=1pt,noitemsep,label=\textbf{A\arabic*}.,start=4,ref={A\arabic*}]
\item \label{as:1} Activation variances are bounded, i.e. there exist constants $c_{L},\
c_{U} > 0$ such that $c_{L} \leq (\tilde{\vec{a}}_{t}^{2}\otimes\vec{1})_{i}
\leq c_{U},\ \forall t > 1,$ and $\forall i \in [d]$ where
$d=m_{\ell}(m_{\ell-1}+1)$. 
\item\label{as:2} For $\eta_{t}\leq \eta_{t-1}$, there exists $t_0 > 0$ such that
$\sqrt{\widehat{\vec{v}}_{t-1}/\widehat{\vec{v}}_{t}}\leq \eta_{t-1}/\eta_{t}$
for $t\geq t_0$. 
\end{enumerate}

\begin{figure}[tb]
\centering
\begin{minipage}{0.48\textwidth}
  \centering
  \includegraphics[width=\linewidth]{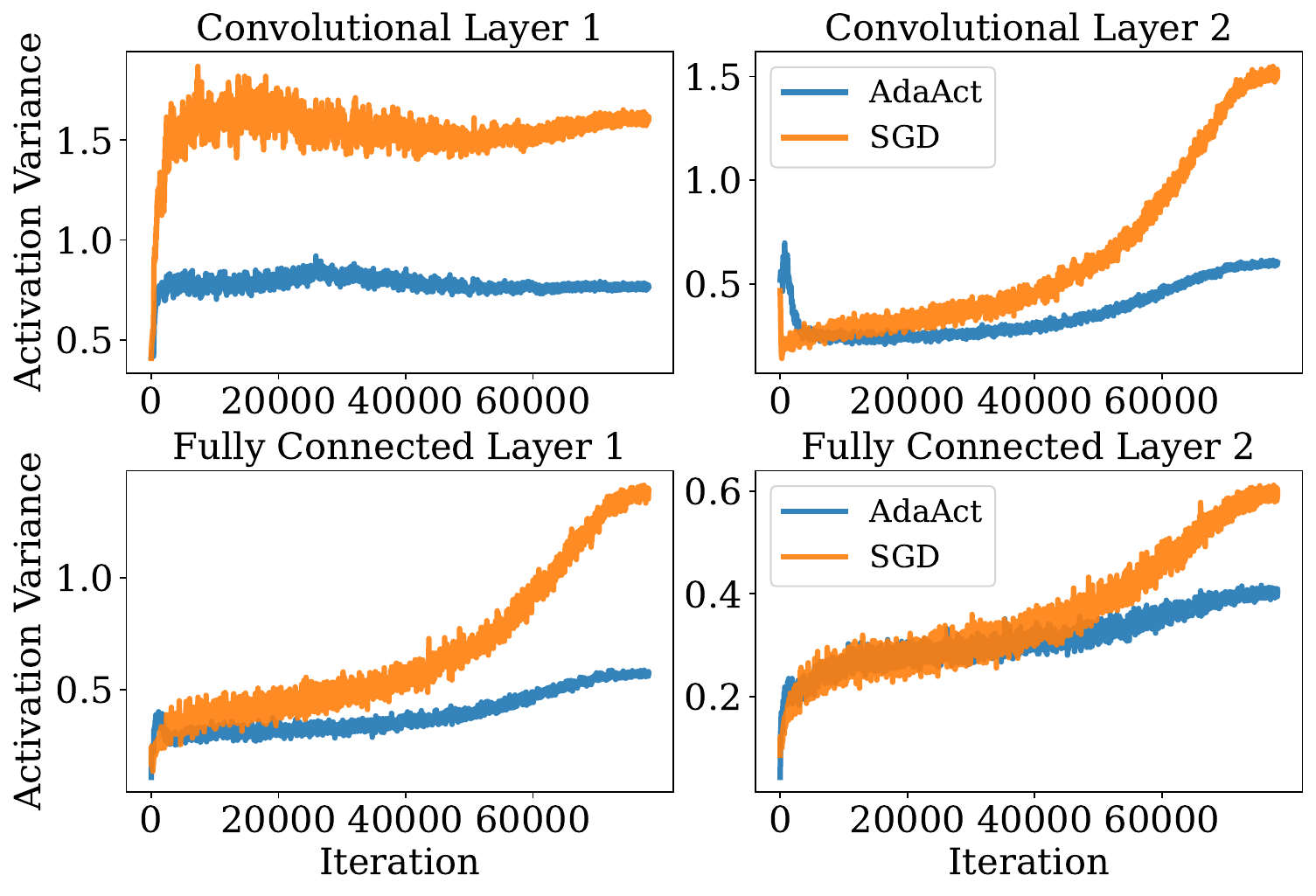}
  \caption{Activation variance resulted from training LeNet-5 on CIFAR10}
  \label{fig:cifar10_lenet5_actvar}
\end{minipage}
\hfill
\begin{minipage}{0.48\textwidth}
  \centering
  \includegraphics[width=\linewidth]{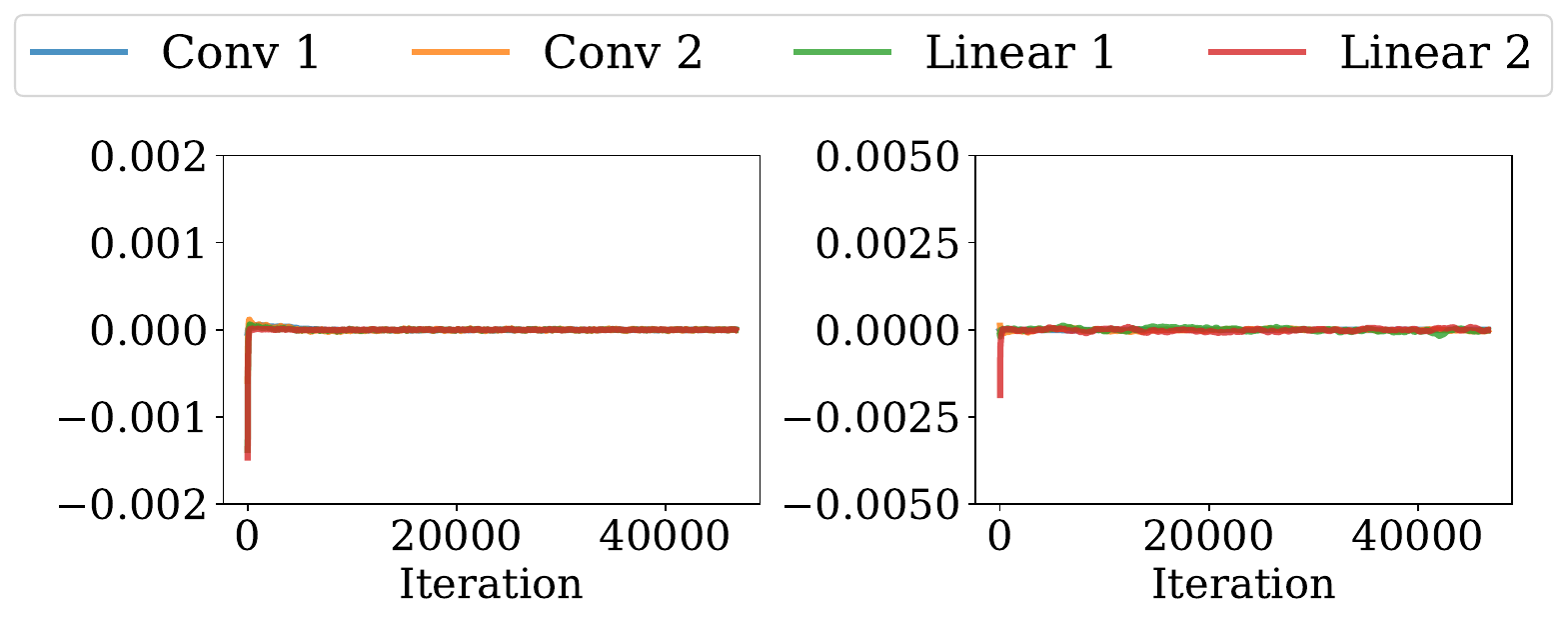}
  \caption{Difference in effective stepsizes: (Left) between iterations $\eta_{t} / \sqrt{\hat{v}_{t}} - \eta_{t-1} / \sqrt{\hat{v}_{t-1}}$, and (Right) between neighboring datasets $\eta / \sqrt{v_{t}} - \eta / \sqrt{v_{t}^{\prime}}$.}
  \label{fig:effstep_diff}
\end{minipage}
\end{figure}

One way to satisfy Assumption~\ref{as:1} is to clip the estimated
activation variances 
$\mathsf{Clip}(\vec{a}_t^2, c_L, c_U)$
to ensure the variances are in $[c_L, c_U]$. However,
we empirically observed that there exist natural lower and upper bound as
\adaact promotes stabilized activations. We trained LeNet-5 on CIFAR10 for 200 epochs to observe the trend of activation variance over iterations. Figure~\ref{fig:cifar10_lenet5_actvar} presents the activation variances across all hidden layers in the architecture. We observe that the activations from layers are bounded.

%
Assumption~\ref{as:2} posits that the
effective learning rates do not increase after a specific iteration
$t_0$. This condition aligns mildly with the inherent behavior of
adaptive methods such as AdaGrad and AMSGrad. 
Figure~\ref{fig:effstep_diff}, generated using LeNet-5 on Fashion MNIST~\cite{xiao2017fashionmnist} illustrates in the left side that 
$\left(\eta_{t}/\sqrt{\widehat{\vec{v}}_t}-\eta_{t-1}/\sqrt{\widehat{\vec{v}}_{t-1}}\right) \to 0$, supporting the validity of Assumption~\ref{as:2}.
Assuming that the assumptions \hyperref[as:A1]{A1}-\hyperref[as:A3]{A3} and
Theorem~\ref{thm:3.1} 
are satisfied, we present the 
following results. 

\begin{corollary}
\label{cor:1}
If Assumption~\ref{as:1} and ~\ref{as:2} hold, for $\beta_{1} \in [0,
1)$ and $\eta_{t}=1/\sqrt{t}$, \adaact satisfies  
\begin{equation}
\label{eq:convergence}
    \underset{t\in[T]}{\min} \E\left[ ||f(\vec{\theta}_{t})||^{2}\right] \leq \frac{1}{\sqrt{T}}\left( Q_{1}+Q_{2}\log{T} \right) 
\end{equation} 
for any $T$, where 
$Q_{1}=\frac{c_{U}}{c_{L}^{2}}\left(C_{1}H^{2}+t_{0}d\left(c_{L}C_{2}+C_{3}\right)\right)+C_{4}$
and $Q_{2}=\frac{c_{U}C_{1}H^{2}}{c_{L}^{2}}$
are two constants independent of $T$.
\end{corollary}
\noindent The result in~\eqref{eq:convergence}, in fact, indicates that \adaact
can achieve the same convergence rate $\mathcal{O}(\log{T}/\sqrt{T})$
as AMSGrad. See Appendix~\ref{proof:cor1} for the proof.

\subsection{Generalization Analysis}
We bound the generalization error of \adaact using the result
of~\cite{hardt2016Train} on a connection between the generalization 
error and stability.
Let $S = (z_1, \ldots, z_N)$ be a set of $N$ i.i.d. samples drawn from 
$\mathcal{D}$. The generalization error of model trained on $S$ using
the randomized algorithm $\mathcal{A}$ is defined as
\[
  \epsilon_{\text{gen}} := \E_{S, \mathcal{A}}
  \left[\mathcal{R}_S(\mathcal{A}(S)) -
    \mathcal{R}(\mathcal{A}(S))\right]\,, 
\]
where $\mathcal{R}_S$ and $\mathcal{R}$ denote the empirical and
population risk, respectively.
\begin{definition}[\cite{hardt2016Train}]
  A randomized algorithm $\mathcal{A}$ is $\epsilon$-uniformly stable
  if for all pairs of datasets $S, S'$ that differ in at most one
  example,
  \[
    \sup_\xi \E_{\mathcal{A}}\left[f(\mathcal{A}(S); \xi) -
      f(\mathcal{A}(S'); \xi)\right] \leq \epsilon\,.
  \]
\end{definition}
\begin{theorem}[\cite{hardt2016Train}] \label{thm:stable_to_gen}
  Let $\mathcal{A}$ be an $\epsilon$-uniformly stable algorithm. Then
  we have $\abs{\epsilon_{\text{gen}}}\leq \epsilon$.
\end{theorem}
Theorem~\ref{thm:stable_to_gen} states that it suffices to prove that \adaact is
$\epsilon$-uniformly stable to bound its generalization error
$\epsilon_{\text{gen}}$. Since the assumption~\ref{as:A2} implies that
the loss function $f$ is $H$-Lipschitz, it remains to show
$\E_{\mathcal{A}}\left[ \norm*{\vec{\theta}_t -
    \vec{\theta}_t'}_2\right]$ is bounded. Then we have
$\sup_\xi \E_{\mathcal{A}}\left[f(\mathcal{A}(S); \xi) -
  f(\mathcal{A}(S'); \xi)\right] \leq H \E_{\mathcal{A}}\left[ \norm*{\vec{\theta}_t -
    \vec{\theta}_t'}_2\right]$.
\begin{theorem} \label{thm:stability}
Let $\vec{\theta}_t$ (or $\vec{\theta}_t'$) be the parameter vector of
model after being trained on $S$ (or $S'$) for $t$ iterations using
\adaact with fixed learning rate $\eta$. Define $\Delta_t:= \norm{\vec{\theta}_t
  -\vec{\theta}_t'}_2$. Then we have
\begin{align*}
  \E\left[ \Delta_{T+1}\right]
  &\leq  \frac{\eta H(N-1)}{N}
    \sum_{t=1}^T\underbracket{\E{\norm*{\frac{1}{\sqrt{\hvec{v}_t}}
    - \frac{1}{\sqrt{\hvec{v}_t'}}}_2}}_{A}\\
  &
    +
    \frac{\eta
    L}{c_L}\sum_{t=1}^T\underbracket{\E\left[\sum_{k=1}^t\beta_1^{t-k}(1-\beta_1)\Delta_k\right]}_{B} 
+ \frac{2\eta HT}{Nc_L}\,.
\end{align*}
\end{theorem}
As shown in Figure~\ref{fig:effstep_diff},
the term $A$ in 
Theorem~\ref{thm:stability} is small enough (almost zero across
iterations). The term $B$ is the EMA of $\Delta_t$ and the last term
is small for datasets of moderate size. The generalization analysis
demonstrates that \adaact maintains a bounded generalization error,
attributable to its $\epsilon$-uniform stability and the Lipschitz
continuity of the loss function. See Appendix~\ref{proof:stability} for the proof.


\section{Experiments}
\label{sec:experiment}
In this section, we evaluate \adaact's performance on the standard
image classification task and compare it with other baselines.  
For comparisons, we  
trained ResNet~\cite{he2016deep}, DenseNet~\cite{Huang2016DenselyCC},
and Vision Transformer (ViT)~\cite{dosovitskiy2021an} 
on standard benchmark datasets: 
CIFAR10, CIFAR100, and ImageNet (ILSVRC 2012)~\cite{deng2009imagenet}.
All experiments were performed using Nvidia Geforce RTX 3090 GPUs. 

\begin{figure}[tb]
    \centering
    \includegraphics[width=0.5\textwidth]{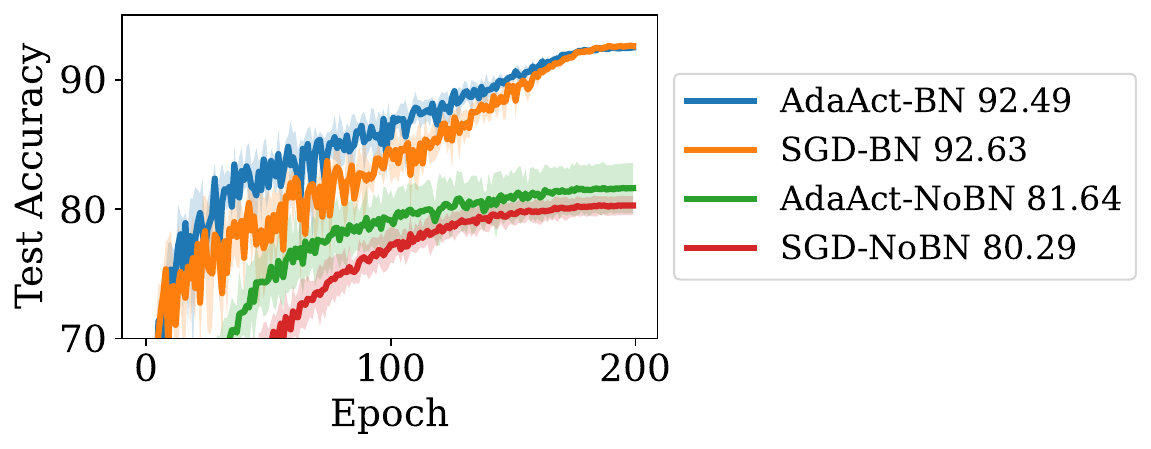}
    \caption{Test accuracy of ResNet-20 on CIFAR10: \adaact
      generalizes better than SGD in the absence of batch
      normalization. }
    \label{fig:bn_removed}
\end{figure}

\subsection{CIFAR Training Results}
\paragraph{Training Settings.} We follow the
settings for training the CIFAR datasets in 
\cite{Luo2019AdaBound} and \cite{zhuang2020adabelief}.
Each network is trained for 200 epochs using the minibatch size of 128
with learning rate decayed according to the cosine annealing schedule.
We used ResNet-20 and ResNet-32 for training CIFAR10 and
ResNet-34 and DenseNet-121 for CIFAR100, and ran the experiments 
5 times and report the mean and standard error for test accuracy
to evaluate the 
generalization performance.  
We included state-of-the-art first- and second-order methods as
baselines. Specifically, we chose SGD as a representative method for
the class of first-order methods, Adam, AdamW, and
Adan~\cite{xie2023adan} for the class of first-order 
adaptive methods, and KFAC and FOOF for the class of second-order methods.
We conducted mild hyperparameter tuning specifically for Adan, FOOF,
and KFAC. We varied the learning rate from 0.001 to 1.0, explored
momentum and EMA coefficient values of 0.9, 0.95, and 0.99, and
adjusted the damping factor between 0.01 and 10. For the remaining
methods, we used the same settings as described in
\cite{zhuang2020adabelief}.

\begin{figure}[tb]
    \centering
    \includegraphics[width=1\textwidth]{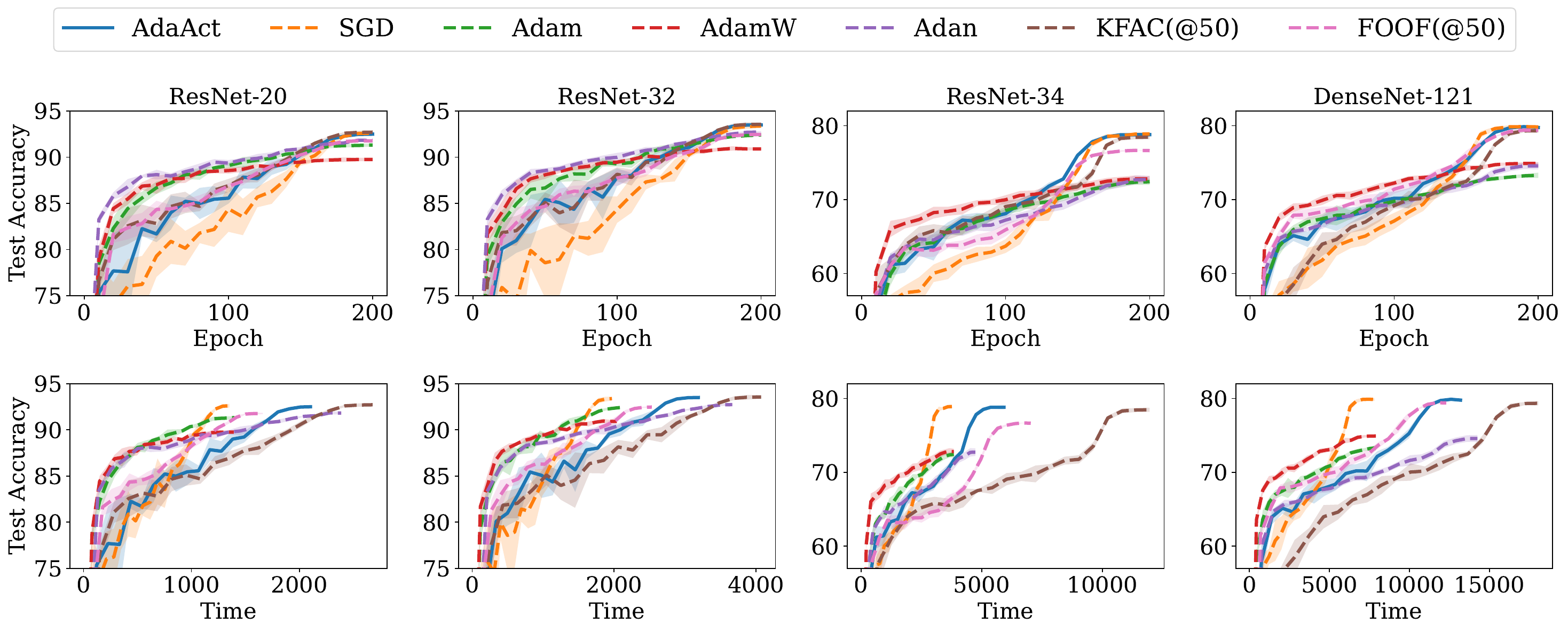}
    \caption{Comparison of test accuracy and training time among methods on CIFAR datasets}
    \label{fig:cifar}
\end{figure}

\begin{table}[tb]
\centering
\caption{Test accuracy (\%) of ResNet and DenseNet on CIFAR datasets} \label{tab:cifar}
\begin{tabular}{c|c|c|c|c}
\toprule
Dataset & \multicolumn{2}{c|}{CIFAR10} & \multicolumn{2}{c}{CIFAR100} \\
Architecture & ResNet-20 & ResNet-32 & ResNet-34 & DenseNet-121\\
\midrule
\adaact & 92.49$\pm$\scriptsize{0.18} & 93.58$\pm$\scriptsize{0.12} & 78.89$\pm$\scriptsize{0.21} & 79.91$\pm$\scriptsize{0.11}\\
SGD & 92.63$\pm$\scriptsize{0.21} & 93.43$\pm$\scriptsize{0.23} & 78.94$\pm$\scriptsize{0.21} & 79.93$\pm$\scriptsize{0.20} \\
Adam & 91.37$\pm$\scriptsize{0.17} & 92.43$\pm$\scriptsize{0.14} & 72.49$\pm$\scriptsize{0.37} & 73.38$\pm$\scriptsize{0.30} \\
AdamW & 89.85$\pm$\scriptsize{0.21} & 90.99$\pm$\scriptsize{0.21} & 72.87$\pm$\scriptsize{0.46} & 74.99$\pm$\scriptsize{0.20} \\
Adan & 91.87$\pm$\scriptsize{0.16} & 92.77$\pm$\scriptsize{0.15} & 72.83$\pm$\scriptsize{0.54} & 74.65$\pm$\scriptsize{0.43} \\
KFAC & 92.74$\pm$\scriptsize{0.24} & 93.64$\pm$\scriptsize{0.20} & 78.51$\pm$\scriptsize{0.32} & 79.45$\pm$\scriptsize{0.27}\\
FOOF & 91.79$\pm$\scriptsize{0.09} & 92.50$\pm$\scriptsize{0.16} & 76.79$\pm$\scriptsize{0.09} & 79.44$\pm$\scriptsize{0.18}\\
\bottomrule
\end{tabular}
\end{table}

\paragraph{Result.}
To demonstrate the effect of activation variance adaption in \adaact,
we trained ResNet-20 
models on CIFAR10 dataset with and without using batch normalization (BN)
and presented the result in Figure~\ref{fig:bn_removed}.
As shown, the removal of BN causes performance degradation for both
\adaact and SGD. However, we see that \adaact is less affected by the
removal. This is because that stabilized activations in \adaact can
create with BN to a certain extent.

Figure~\ref{fig:cifar} shows the 
test accuracy of 
methods against iterations and wall-clock time. As shown in the graphs on the
top row, the 
methods belonging to the adaptive gradient family
(i.e., Adan, Adam, and AdamW) achieve higher test accuracy at early
epochs and quickly reach the plateau around epoch 100.  
However, at the end of training, they end up achieving lower test
accuracy than the other three algorithms: \adaact, SGD, and KFAC
(see Figure \ref{fig:cifar} and Table \ref{tab:cifar}).
This coincides with the observation made in prior work that 
adaptive gradient methods
are faster in terms of convergence but
suffer from poor generalization. 
While \adaact is an adaptive  
methods, it achieves similar 
test accuracy with SGD and KFAC.
This demonstrates the effectiveness of \adaact's
activation variance-based adaptation in improving generalization
performance.  
The graphs on the bottom row of Figure~\ref{fig:cifar} shows that KFAC
achieves the same accuracy with SGD and \adaact, but it's the slowest
in terms of wall-clock time.
FOOF and KFAC  
require computing the inverse of preconditioning matrix periodically
and its frequency is controlled by the hyperparameter
$T_{\text{inv}}$. We set $T_{\text{inv}}=50$ for both FOOF and KFAC.
For CIFAR100, we observe that 
\adaact converges as fast as Adam with training time similar to that
of Adan --- it is still significantly faster than FOOF and
KFAC. 
This shows that \adaact has an ability to match the
generalization performance of KFAC while its speed is comparable to
that of adaptive first-order algorithms.

\subsection{ImageNet Training Results}
\paragraph{Training Settings.} For ImageNet,
we train ResNet-50, 101, and ViT-S networks, adopting the ``A2''
settings 
described in~\cite{wightman2021resnet}. It utilizes random crop,
horizontal flip, Mixup (0.1) \cite{zhang2018mixup}/CutMix (1.0)
\cite{Yun2019CutMixRS} with probability 0.5, and RandAugment
\cite{Cubuk2019RandaugmentPA} with $M=7$, $N=2$ and $MSTD=0.5$. It
employs stochastic depth \cite{Huang2016DeepNW} set at 0.05 and
utilizes a cosine learning rate decay, in conjunction with a binary
cross-entropy loss. 
For \adaact, we used a template code from \cite{rw2019timm}, setting a
mini-batch size of 2,048 and cross-entropy loss is used for all architectures.
We compare \adaact with the baselines as previously reported by
\cite{xie2023adan}, but we omit the training results from Adan as they
rely on micro fine-tuning. For both 
architectures, we used a
large learning rate of 4.0, following \emph{the linear scaling rule}
as suggested in~\cite{Goyal2017AccurateLM}. We opted for a smaller
weight decay value in ViT compared to ResNets to facilitate faster
convergence, as its gradient per iteration significantly differs from
that of CNNs due to a much sharper loss landscape
\cite{Chen2021WhenVT}. 

\begin{table}[tb]
\centering
\caption{Top-1 accuracy (\%) of ResNet-50 and 101 on ImageNet for 100
epochs. $^\dagger$ are reported in {\protect\cite{xie2023adan}}}
\label{tab:imgnet_resnet}
\begin{tabular}{c|c|c}
    \toprule
    Architecture & ResNet-50 & ResNet-101 \\
    \midrule
    \adaact & 77.6 & 79.4  \\
    SGD & 77.0$^\dagger$  & 79.3$^\dagger$ \\
    Adam & 76.9$^\dagger$ & 78.4$^\dagger$ \\
    AdamW & 77.0$^\dagger$  & 78.9$^\dagger$ \\
    LAMB & 77.0$^\dagger$  & 79.4$^\dagger$ \\
    SAM & 77.3$^\dagger$  & 79.5$^\dagger$ \\
    \bottomrule
\end{tabular}
\end{table}

\begin{table}[tb]
\caption{Top-1 accuracy (\%) of ViT-S on ImageNet for 150
epochs. $^\dagger$ are reported in {\protect\cite{xie2023adan}}} \label{tab:imgnet_vit}
\begin{center}
\begin{tabular}{c|cccc}
    \toprule
     \adaact & SGD & Adam & AdamW & LAMB \\
    \midrule
    73.8 & 68.7$^\dagger$ & 64.0$^\dagger$ & 78.9$^\dagger$ & 73.8$^\dagger$\\
    \bottomrule
\end{tabular}
\end{center}
\end{table}

\paragraph{Result.}
Table~\ref{tab:imgnet_resnet} demonstrate that \adaact can provide
good performance in large batch training setup (used in large-scale
training).
Specifically, \adaact achieves the
top-1 accuracy of 77.6\% on ResNet-50, 
higher than
other baseline methods, most of them showing the accuracy around 77.0\%.
For ResNet-101, 
\adaact delivers competitive accuracy of 79.4\%, matching the
accuracy of LAMB~\cite{You2020Large}
and is only slightly behind the
performance of SAM~\cite{foret2021sharpnessaware} (79.5\%),
the best performer in this comparison but the slowest at the same time
(due to the use of twice as many backprops as other methods).
The fact that \adaact surpasses Adam and LAMB, the methods-of-choice
in practice for large-batch training, indicates its potential as
an alternative 
in large-scale training. 
Table~\ref{tab:imgnet_vit} presents the top-1 accuracy of ViT-S model
on ImageNet dataset.
In this experiment, \adaact attains the top-1 accuracy of 73.8\%, 
matching LAMB's performance, which is significant given that LAMB is
specifically designed for this setup and for
Transformers~\cite{Vaswani2017AttentionIA}.  
Although \adaact does not achieve the same accuracy with AdamW's leading 78.9\%, it still 
surpasses traditional methods such as SGD and Adam. This
demonstrates
\adaact's suitability 
and ability to handle particular optimization challenges for vision
transformers.
The fact that \adaact
outperforms SGD and Adam underscores its capability in
navigating ViTs' complex optimization landscape, which is notably
different from that of CNNs.
\adaact's comparable performance to LAMB, while still showing some gap
from AdamW's best performance, nevertheless marks it as a versatile
optimization method potentially applicable across various
architectures.


\section{Conclusions}
\label{sec:conclusion}
We presented \adaact, an adaptive method designed to achieve improved
generalization via stabilizing neuron outputs.
Our approach focuses on adaptivity at the neuron
level, promoting stable neuron responses even 
in the presence of varying activation variances.
Beyond enhanced generalization, \adaact introduces a fresh
perspective on adapting learning rates based on activation variance, complementing existing activation regularization methods. In conclusion, \adaact offers an effective solution to the challenges
associated with adaptive optimization methods. Its improvements in
generalization and network stability make it a valuable addition to
the toolkit of deep learning practitioners.


\bibliographystyle{plain}

\newpage
\appendix
\addcontentsline{toc}{section}{Appendix}

\section{Proof of Theorem~\ref{thm:3.1}}

Notice that \adaact falls within the class of general
Adam-type optimizer described in Algorithm~\ref{alg:genadam}. To see
this, we rewrite the $\mat{V}_t$ update in Line~\ref{alg:v_update} of
Algorithm~\ref{alg:adaact} in a vector form.
$\widehat{\vec{v}}_{t}
= \beta_{2}\widehat{\vec{v}}_{t-1} +
(1-\beta_{2})(\tilde{\vec{a}}_{t}^{2}\otimes\vec{1}_{m_{\ell}})$, where
$\tilde{\vec{a}}_{t}^{2}\otimes\vec{1}_{m_{\ell}} =
\vec{g}_{t}^{2}/(\vec{1}_{(m_{\ell-1}+1)}\otimes \vec{p}_{t}^{2})$.

\label{proof:thm3.1}
\begin{algorithm}[ht]
\caption{Generalized Adam}
\label{alg:genadam}
\textbf{Initialize} $\vec{m}_{0} = 0$ and $\vec{\theta}_{1}$ 
\begin{algorithmic}[1]
\FOR{$t = 1$ to $T$}
    \STATE $\vec{m}_{t} = \beta_{1,t} \vec{m}_{t-1} + (1 - \beta_{1,t}) \vec{g}_{t}$ 
    \STATE $\widehat{\vec{v}}_{t} = h_k(\vec{g}_{1}, \vec{g}_{2}, \dots, \vec{g}_{t})$ 
    \STATE $\vec{\theta}_{t+1} = \vec{\theta}_{t} - \eta_{t} \frac{\vec{m}_{t}}{\sqrt{\widehat{\vec{v}}_{t}}}$ 
\ENDFOR
\end{algorithmic}
\end{algorithm}

\begin{lemma} \label{lem:zseq_diff}
  Let $\vec{\theta}_0  := \vec{\theta}_1$ in
  Algorithm~\ref{alg:genadam}. Consider the sequence
  \[
    \vec{z}_t = \vec{\theta}_t + \frac{\beta_1}{1 - \beta_1}(\vec{\theta} -
    \vec{\theta}_{t-1})\,, \quad \forall t \geq 1\,.
  \]
  Then we have
  \begin{align*}
    \vec{z}_{t+1} -\vec{z}_t
    &= -\frac{\beta_1}{1-\beta_1}
      \left(\frac{\eta_t}{\sqrt{\hvec{v}_t}} - \frac{\eta_{t-1}}{\sqrt{\hvec{v}_{t-1}}}\right) \odot
      \vec{m}_{t-1}-\eta_t\vec{g}_t/\sqrt{\hvec{v}_t}\,, \qquad
      \forall t>1\,,\\
    \intertext{and}
    \vec{z}_2 - \vec{z}_1
    &= -\frac{\eta_1\vec{m}_1}{(1-\beta_1)\sqrt{\hvec{v}_1}}
      = -\frac{\eta_1\vec{g}_1}{\sqrt{\hvec{v}_1}}\,.    
  \end{align*}
\end{lemma}
\begin{proof}
  \begin{align*}
    \vec{\theta}_{t+1} - \vec{\theta}_t
    &= -\eta_t\vec{m}_t / \sqrt{\hvec{v}_t}\\
    &= -\eta_t(\beta_1\vec{m}_{t-1} + (1-\beta_1)\vec{g}_t) \sqrt{\hvec{v}_t}\\
    &= \beta_1\frac{\eta_t}{\eta_{t-1}}\frac{\sqrt{\hvec{v}_{t-1}}}{\sqrt{\hvec{v}_t}}\odot
      (\vec{\theta}_{t}-\vec{\theta}_{t-1}) -\eta_t(1-\beta_1)\vec{g}_t/\sqrt{\hvec{v}_t}\\
    &= \beta_1(\vec{\theta}_{t}-\vec{\theta}_{t-1}) + \beta_1\left(
      \frac{\eta_t}{\eta_{t-1}}\frac{\sqrt{\hvec{v}_{t-1}}}{\sqrt{\hvec{v}_t}}
      - 1 \right) \odot (\vec{\theta}_{t}-\vec{\theta}_{t-1})
      -\eta_t(1-\beta_1)\vec{g}_t/\sqrt{\hvec{v}_t}\\
    &= \beta_1(\vec{\theta}_{t}-\vec{\theta}_{t-1}) + \beta_1\left(
      \frac{\eta_t}{\eta_{t-1}}\frac{\sqrt{\hvec{v}_{t-1}}}{\sqrt{\hvec{v}_t}}
      - 1 \right) \odot (-\eta_{t-1}\frac{\vec{m}_{t-1}}{\sqrt{\hvec{v}_{t-1}}})
      -\eta_t(1-\beta_1)\vec{g}_t/\sqrt{\hvec{v}_t}\\
    &=\beta_1(\vec{\theta}_{t}-\vec{\theta}_{t-1}) - \beta_1\left(
      \frac{\eta_t}{\sqrt{\hvec{v}_t}} - \frac{\eta_{t-1}}{\sqrt{\hvec{v}_{t-1}}}
      \right) \odot
      \vec{m}_{t-1}-\eta_t(1-\beta_1)\vec{g}_t/\sqrt{\hvec{v}_t}
  \end{align*}
  Since $\vec{\theta}_{t+1} - \vec{\theta}_t =
  (1-\beta_1)(\vec{\theta}_{t+1}-\vec{\theta}_t) +
  \beta_1(\vec{\theta}_{t+1}-\vec{\theta}_t) = (1-\beta_1)\vec{\theta}_{t+1} +
  \beta_1(\vec{\theta}_{t+1}-\vec{\theta}_t) - (1-\beta_1)\vec{\theta}_t$, we have
  \begin{align*}
    & (1-\beta_1)\vec{\theta}_{t+1} + \beta_1(\vec{\theta}_{t+1}-\vec{\theta}_t)\\
    &= (1-\beta_1)\vec{\theta}_t + 
      \beta_1(\vec{\theta}_{t}-\vec{\theta}_{t-1}) - \beta_1\left(
      \frac{\eta_t}{\sqrt{\hvec{v}_t}} - \frac{\eta_{t-1}}{\sqrt{\hvec{v}_{t-1}}}
      \right) \odot
      \vec{m}_{t-1}-\eta_t(1-\beta_1)\vec{g}_t/\sqrt{\hvec{v}_t}\\
    \intertext{Dividing both sides by $1-\beta_1$ yields }
    &\vec{\theta}_{t+1} + \frac{\beta_1}{1-\beta_1}(\vec{\theta}_{t+1}-\vec{\theta}_t) \\
    &=\vec{\theta}_t + \frac{\beta_1}{1-\beta_1}(\vec{\theta}_{t}-\vec{\theta}_{t-1})      
      - \frac{\beta_1}{1-\beta_1}
      \left(\frac{\eta_t}{\sqrt{\hvec{v}_t}} - \frac{\eta_{t-1}}{\sqrt{\hvec{v}_{t-1}}}\right) \odot
      \vec{m}_{t-1}-\eta_t\vec{g}_t/\sqrt{\hvec{v}_t}\,.
  \end{align*}
  Define the sequence
  \[
    \vec{z}_t = \vec{\theta}_t + \frac{\beta_1}{1-\beta_1}(\vec{\theta}_{t}-\vec{\theta}_{t-1})\,.
  \]
  Then we have
  \begin{align*}
    \vec{z}_{t+1}
    &= \vec{z}_t - \frac{\beta_1}{1-\beta_1}
      \left(\frac{\eta_t}{\sqrt{\hvec{v}_t}} - \frac{\eta_{t-1}}{\sqrt{\hvec{v}_{t-1}}}\right) \odot
      \vec{m}_{t-1}-\eta_t\vec{g}_t/\sqrt{\hvec{v}_t}\,, \forall t>1\,.
  \end{align*}
  For $t=1$, we have $\vec{z}_1 = \vec{\theta}_1$, and
  \begin{align*}
    \vec{z}_2 - \vec{z}_1
    &= \vec{\theta}_2 + \frac{\beta_1}{1-\beta_1}(\vec{\theta}_2-\vec{\theta}_1) - \vec{\theta}_1\\
    &= \frac{1}{1-\beta_1}(\vec{\theta}_2-\vec{\theta}_1)\\
    &= -\frac{\eta_1\vec{m}_1}{(1-\beta_1)\sqrt{\hvec{v}_1}}
      = -\frac{\eta_1\vec{g}_1}{\sqrt{\hvec{v}_1}}\,.
  \end{align*}  
\end{proof}

Without loss of generality, we assume that Algorithm~\ref{alg:genadam}
is initialized such that
\begin{equation} \label{eq:algo_init}
  \left(\frac{\eta_1}{\sqrt{\hvec{v}_1}} -
    \frac{\eta_0}{\hvec{v}_0}\right) \odot \vec{m}_0 = \vec{0}\,.
\end{equation}

\begin{lemma} \label{lem:opt_gap}
  Suppose that assumptions~\ref{as:A1}-~\ref{as:A3} hold
  true. Then we have
  \[
    \E{f(\vec{z}_{t+1}) - f(\vec{z}_1)}\leq \sum_{i=1}^4 T_i\,,
  \]
  where
  \begin{align}
    T_1
    &= -\E{\sum_{i=1}^t
    \ip*{\grad{f}(\vec{z}_i),\, \frac{\beta_1}{1-\beta_1}
    \left(\frac{\eta_i}{\sqrt{\hvec{v}_i}} - \frac{\eta_{i-1}}{\sqrt{\hvec{v}_{i-1}}}\right) \odot
      \vec{m}_{i-1}}}\,, \\
    T_2
    &= -\E{\sum_{i=1}^t \ip*{\grad{f}(\vec{z}_i),\, \eta_t\vec{g}_i/\sqrt{\hvec{v}_i}}}\,,\\
    T_3
    &= \E{\sum_{i=1}^t L\norm*{\frac{\beta_1}{1-\beta_1}
    \left(\frac{\eta_i}{\sqrt{\hvec{v}_i}} - \frac{\eta_{i-1}}{\sqrt{\hvec{v}_{i-1}}}\right) \odot
      \vec{m}_{i-1}}^2}\,,\\
    T_4
    &= \E{\sum_{i=1}^tL\norm*{\frac{\eta_i\vec{g}_i}{\sqrt{\hvec{v}_i}}}^2}\,.
  \end{align}
\end{lemma}
\begin{proof}
By the smoothness of $\grad{f}$, we have
\[
  f(\vec{z}_{t+1}) \leq f(\vec{z}_t) + \ip{\grad{f}(\vec{z}_t),\,
    \vec{d}_t} + \frac{L}{2}\norm{\vec{d}_t}^2\,,
\]
where $\vec{d}_t = \vec{z}_{t+1} - \vec{z}_t$.
\begin{align*}
  f(\vec{z}_{i+1}) - f(\vec{z}_i)
  &\leq \ip{\grad{f}(\vec{z}_i),\,\vec{d}_i} +
    \frac{L}{2}\norm{\vec{d}_i}^2 \\
  \intertext{From Lemma~\ref{lem:zseq_diff}, we have }
  &=-\ip*{\grad{f}(\vec{z}_i),\, \frac{\beta_1}{1-\beta_1}
    \left(\frac{\eta_t}{\sqrt{\hvec{v}_t}} - \frac{\eta_{t-1}}{\sqrt{\hvec{v}_{t-1}}}\right) \odot
    \vec{m}_{t-1}}
    -\ip*{\grad{f}(\vec{z}_i),\, \eta_t\vec{g}_t/\sqrt{\hvec{v}_t}} +
    \frac{L}{2}\norm{\vec{d}_i}^2\\
  \intertext{From the above, we get}
  \E{f(\vec{z}_{t+1}) - f(\vec{z}_1)}
  &=\E{\sum_{i=1}^t f(\vec{z}_{i+1}) - f(\vec{z}_i)}\\
  &\leq -\E{\sum_{i=1}^t
    \ip*{\grad{f}(\vec{z}_i),\, \frac{\beta_1}{1-\beta_1}
    \left(\frac{\eta_i}{\sqrt{\hvec{v}_i}} - \frac{\eta_{i-1}}{\sqrt{\hvec{v}_{i-1}}}\right) \odot
    \vec{m}_{i-1}}}    \\
  &\qquad  -\E{\sum_{i=1}^t \ip*{\grad{f}(\vec{z}_i),\, \eta_t\vec{g}_i/\sqrt{\hvec{v}_i}}}   \\
  & \qquad + \E{\sum_{i=1}^t \frac{L}{2}\norm*{-\frac{\beta_1}{1-\beta_1}
    \left(\frac{\eta_t}{\sqrt{\hvec{v}_i}} - \frac{\eta_{i-1}}{\sqrt{\hvec{v}_{i-1}}}\right) \odot
    \vec{m}_{i-1}- \frac{\eta_t\vec{g}_i}{\sqrt{\hvec{v}_i}}}^2}\\
  &\leq -\E{\sum_{i=1}^t
    \ip*{\grad{f}(\vec{z}_i),\, \frac{\beta_1}{1-\beta_1}
    \left(\frac{\eta_i}{\sqrt{\hvec{v}_i}} - \frac{\eta_{i-1}}{\sqrt{\hvec{v}_{i-1}}}\right) \odot
    \vec{m}_{i-1}}}    \\
  &\qquad  -\E{\sum_{i=1}^t \ip*{\grad{f}(\vec{z}_i),\, \eta_t\vec{g}_i/\sqrt{\hvec{v}_i}}}   \\
  & \qquad + \E{\sum_{i=1}^t L\norm*{\frac{\beta_1}{1-\beta_1}
    \left(\frac{\eta_i}{\sqrt{\hvec{v}_i}} - \frac{\eta_{i-1}}{\sqrt{\hvec{v}_{i-1}}}\right) \odot
    \vec{m}_{i-1}}^2} +
    \E{\sum_{i=1}^tL\norm*{\frac{\eta_i\vec{g}_i}{\sqrt{\hvec{v}_i}}}^2} \\
  &=\red{T_1} + \green{T_2} + \yellow{T_3} + \blue{T_4}\,,
\end{align*}
where the last inequality is due to $\norm{\vec{a} + \vec{b}}^2\leq 2\norm{\vec{a}}^2 + 2\norm{\vec{b}}^2$.
\end{proof}

Now, in the following series of lemmas, we bound each term in the
above separately.

\begin{lemma}
  Under the assumptions of \hyperref[as:A1]{A1}-~\hyperref[as:A3]{A3}, we have
  \begin{align*}
    \red{T_1}
    &= -\E{\sum_{i=1}^t
      \ip*{\grad{f}(\vec{z}_i),\, \frac{\beta_1}{1-\beta_1}
      \left(\frac{\eta_t}{\sqrt{\hvec{v}_t}} - \frac{\eta_{t-1}}{\sqrt{\hvec{v}_{t-1}}}\right) \odot
      \vec{m}_{t-1}}} \\
    &\leq H^2\frac{\beta_1}{1-\beta_1}\E{\sum_{i=2}^t\sum_{j=1}^d
      \abs*{\left(\frac{\eta_i}{\sqrt{\hvec{v}_i}} - \frac{\eta_{i-1}}{\sqrt{\hvec{v}_{i-1}}}\right)_j}}\,.
  \end{align*}
\end{lemma}
\begin{proof}
  By the assumption~\hyperref[as:A2]{A2}, we have
  $\norm{\vec{g}_t}\leq H$. Since $\vec{m}_t = \beta_1\vec{m}_{t-1} +
  (1-\beta_1)\vec{g}_t$, we have $\norm{\vec{m}_t}\leq H$ (this can
  proved using a simple induction).
  \begin{align*}
    T_1
    &= -\E{\sum_{i=2}^t
    \ip*{\grad{f}(\vec{z}_i),\, \frac{\beta_1}{1-\beta_1}
      \left(\frac{\eta_t}{\sqrt{\hvec{v}_t}} - \frac{\eta_{t-1}}{\sqrt{\hvec{v}_{t-1}}}\right) \odot
      \vec{m}_{t-1}}} \\
    &\leq \E{\sum_{i=2}^t \norm{\grad{f}(\vec{z}_i)}
      \norm{\vec{m}_{i-1}} \frac{\beta_1}{1-\beta_1}
      \sum_{j=1}^d\abs*{\left(\frac{\eta_i}{\sqrt{\hvec{v}_i}} -
      \frac{\eta_{i-1}}{\sqrt{\hvec{v}_{i-1}}}\right)_j}} \\
    &\leq H^2\frac{\beta_1}{1-\beta_1}
      \E{\sum_{i=2}^t \sum_{j=1}^d\abs*{\left(\frac{\eta_i}{\sqrt{\hvec{v}_i}} -
      \frac{\eta_{i-1}}{\sqrt{\hvec{v}_{i-1}}}\right)_j}}\,.
  \end{align*}
  In this above, we applied the Cauchy-Schwarz inequality to the first inequality, and the
  second inequality is due to the assumption of bounded gradient.
\end{proof}

\begin{lemma}
  Suppose the assumptions~\hyperref[as:A1]{A1}-~\hyperref[as:A3]{A3} hold. Then
  we have
  \begin{align*}
    \green{T_2}
    &=-\E{\sum_{i=1}^t \ip*{\grad{f}(\vec{z}_i),\, \eta_i\vec{g}_i/\sqrt{\hvec{v}_i}}}\\
    &\leq L^2\left(\frac{\beta_1}{1-\beta_1}\right)^2
    \E{\sum_{i=1}^{t-1}\norm*{\frac{\eta_i\vec{g}_i}{\sqrt{\hvec{v}_i}}}^2}
    + L^2H^2\left(\frac{\beta_1}{1-\beta_1}\right)^4
    \E{\sum_{j=1}^d\sum_{l=2}^{t-1}\abs*{\frac{\eta_{l}}{\sqrt{\hvec{v}_{l}}} -
      \frac{\eta_{l-1}}{\sqrt{\hvec{v}_{l-1}}} }_j^2} \\
    &\qquad +2H^2\E{\sum_{i=2}^t\sum_{j=1}^d
      \abs*{\left( \frac{\eta_i}{\sqrt{\hvec{v}_i}}\right)_j -
      \left(\frac{\eta_{i-1}}{\sqrt{\hvec{v}_{i-1}}}\right)_j}}
      +
      2H^2\E{\sum_{j=1}^d\left(\frac{\eta_1}{\hvec{v}_1}\right)_j}
      \nonumber \\
    &\qquad - \E{\sum_{i=1}^t \eta_i\ip{\grad{f}(\vec{\theta}_i),\,
      \grad{f}(\vec{\theta}_i)/\sqrt{\hvec{v}_i}}}      \\
    &\qquad +
      \frac{1}{2}\E{\sum_{i=2}^t\norm*{\frac{\eta_i\vec{g}_i}{\sqrt{\hvec{v}_i}}}^2}\,.
  \end{align*}  
\end{lemma}
\begin{proof}
  From the definition, we have
  \[
    \vec{z}_i - \vec{\theta}_i = \frac{\beta_1}{1-\beta_1}(\vec{\theta}_i -
    \vec{\theta}_{i-1}) =
    \frac{\beta_1}{1-\beta_1}\left(-\frac{\eta_{i-1}\vec{m}_{i-1}}{\sqrt{\hvec{v}_{i-1}}}\right)
    \text{ and } \vec{z}_1 = \vec{\theta}_1\,.
  \]
  Thus, we have
  \begin{align*}
    T_2
    &= -\E{\sum_{i=1}^t \eta_i\ip*{\grad{f}(\vec{z}_i) +
      \grad{f}(\vec{\theta}_i) - \grad{f}(\vec{\theta}_i),\,
      \frac{\vec{g}_i}{\sqrt{\hvec{v}_i}}}}\\
    &=-\E{\sum_{i=1}^t \eta_i\ip*{\grad{f}(\vec{\theta}_i), \, \frac{\vec{g}_i}{\sqrt{\hvec{v}_i}}}}
      -\E{\sum_{i=1}^t
      \eta_i\ip*{\grad{f}(\vec{z}_i)-\grad{f}(\vec{\theta}_i), \,
      \frac{\vec{g}_i}{\sqrt{\hvec{v}_i}}}}\,. \\
    \intertext{Applying $\ip{\vec{a}, \vec{b}}\leq
    \frac{1}{2}(\norm{\vec{a}}^2 +\norm{\vec{b}}^2)$ to the second
    term yields}
    &\leq -\E{\sum_{i=1}^t \eta_i\ip*{\grad{f}(\vec{\theta}_i), \,
      \frac{\vec{g}_i}{\sqrt{\hvec{v}_i}}}}
      + \E{\sum_{i=2}^t
      \frac{1}{2}\norm{\grad{f}(\vec{\theta}_i)-\grad{f}(\vec{z}_i)}^2 +
      \frac{1}{2}\norm*{\frac{\eta_i\vec{g}_i}{\sqrt{\hvec{v}_i}}}^2}\,.\\
    \intertext{From the smoothness of $\grad{f}$, we get}
    &\leq -\E{\sum_{i=1}^t \eta_i\ip*{\grad{f}(\vec{\theta}_i), \,
      \frac{\vec{g}_i}{\sqrt{\hvec{v}_i}}}}
      +
      \frac{L^2}{2}\E{\sum_{i=2}^t\norm*{\frac{\beta_1}{1-\beta_1}\left(\frac{\eta_{i-1}\vec{m}_{i-1}}{\sqrt{\hvec{v}_{i-1}}}\right)}^2}
      +
      \frac{1}{2}\E{\sum_{i=2}^t\norm*{\frac{\eta_i\vec{g}_i}{\sqrt{\hvec{v}_i}}}^2}
    \\
    &= -T_5 + \frac{L^2}{2}T_6 + \frac{1}{2}T_7\,.
  \end{align*}
  \paragraph{Bound on $T_5$.}
  The noisy gradient can be expressed as $\vec{g}_t
  =\grad{f}(\vec{\theta}_t) + \vec{\delta}_t$ with $\E{\vec{\delta}_t} =
  \vec{0}$.
  \begin{align}
    T_5
    &=\E{\sum_{i=1}^t \eta_i\ip{\grad{f}(\vec{\theta}_i),\,
      \vec{g}_i/\sqrt{\hvec{v}_i}}}\\
    &=\E{\sum_{i=1}^t \eta_i\ip{\grad{f}(\vec{\theta}_i),\,
      (\grad{f}(\vec{\theta}_i) + \vec{\delta}_i)/\sqrt{\hvec{v}_i}}}\\
    &=\E{\sum_{i=1}^t \eta_i\ip{\grad{f}(\vec{\theta}_i),\,
      \grad{f}(\vec{\theta}_i)/\sqrt{\hvec{v}_i}}}
      + \E{\sum_{i=1}^t\eta_i\ip{\grad{f}(\vec{\theta}_i),\,
      \vec{\delta}_i/\sqrt{\hvec{v}_i}}} \label{eq:T5_decomp}
  \end{align}
  The second term in~\eqref{eq:T5_decomp} can be bounded as follows.
  \begin{align}
    &\E{\sum_{i=1}^t\eta_i\ip{\grad{f}(\vec{\theta}_i),\,
    \vec{\delta}_i/\sqrt{\hvec{v}_i}}} \nonumber \\
    &= \E{\sum_{i=2}^t\ip*{\grad{f}(\vec{\theta}_i),\,  \vec{\delta}_i\odot
      \left( \frac{\eta_i}{\sqrt{\hvec{v}_i}} -
      \frac{\eta_{i-1}}{\sqrt{\hvec{v}_{i-1}}} +
      \frac{\eta_{i-1}}{\sqrt{\hvec{v}_{i-1}}} \right)}} +
      \E{\eta_1\ip*{\grad{f}(\vec{\theta}_1),\,
      \frac{\vec{\delta}_1}{\hvec{v}_1}}}  \nonumber \\
    &\geq \E{\sum_{i=2}^t\ip*{\grad{f}(\vec{\theta}_i),\,  \vec{\delta}_i\odot
      \left( \frac{\eta_i}{\sqrt{\hvec{v}_i}} -
      \frac{\eta_{i-1}}{\sqrt{\hvec{v}_{i-1}}} \right)}} +
      \E{\sum_{i=2}^t \ip*{\grad{f}(\vec{\theta}_i),\,  \vec{\delta}_i\odot
      \frac{\eta_{i-1}}{\sqrt{\hvec{v}_{i-1}}}}} -
      2H^2\E{\sum_{j=1}^d\left(\frac{\eta_1}{\hvec{v}_1}\right)_j}  \nonumber \\
    \intertext{Notice that given $\vec{\theta}_i, \hvec{v}_{i-1}$ the
    expectation in the second term is 0.}
    &=\E{\sum_{i=2}^t\ip*{\grad{f}(\vec{\theta}_i),\,  \vec{\delta}_i\odot
      \left( \frac{\eta_i}{\sqrt{\hvec{v}_i}} -
      \frac{\eta_{i-1}}{\sqrt{\hvec{v}_{i-1}}} \right)}} -
      2H^2\E{\sum_{j=1}^d\left(\frac{\eta_1}{\hvec{v}_1}\right)_j} \label{eq:bound_T5_1}
  \end{align}
  The first term in~\eqref{eq:T5_decomp} can be bounded as
  \begin{align}
    &\E{\sum_{i=2}^t\ip*{\grad{f}(\vec{\theta}_i),\,  \vec{\delta}_i\odot
      \left( \frac{\eta_i}{\sqrt{\hvec{v}_i}} -
      \frac{\eta_{i-1}}{\sqrt{\hvec{v}_{i-1}}} \right)}} \nonumber \\
    &=\E{\sum_{i=2}^t\sum_{j=1}^d (\grad{f}(\vec{\theta}_i))_j \cdot
      (\vec{\delta}_i)_j \cdot 
      \left(\left( \frac{\eta_i}{\sqrt{\hvec{v}_i}}\right)_j -
      \left(\frac{\eta_{i-1}}{\sqrt{\hvec{v}_{i-1}}}\right)_j\right)} \nonumber \\
    &\geq -
      \E{\sum_{i=2}^t\sum_{j=1}^d \abs*{(\grad{f}(\vec{\theta}_i))_j} \cdot
      \abs*{(\vec{\delta}_i)_j} \cdot 
      \abs*{\left( \frac{\eta_i}{\sqrt{\hvec{v}_i}}\right)_j -
      \left(\frac{\eta_{i-1}}{\sqrt{\hvec{v}_{i-1}}}\right)_j}}  \nonumber \\
    &\geq -2H^2\E{\sum_{i=2}^t\sum_{j=1}^d
      \abs*{\left( \frac{\eta_i}{\sqrt{\hvec{v}_i}}\right)_j -
      \left(\frac{\eta_{i-1}}{\sqrt{\hvec{v}_{i-1}}}\right)_j}}\,. \label{eq:bound_T5_2}
  \end{align}
  Applying~\eqref{eq:bound_T5_2} and~\eqref{eq:bound_T5_1}
  to~\eqref{eq:T5_decomp} gives
  \begin{align}
    -T_5 = -\E{\sum_{i=1}^t \eta_i\ip*{\grad{f}(\vec{\theta}_i), \,
    \frac{\vec{g}_i}{\sqrt{\hvec{v}_i}}}}
    &\leq 2H^2\E{\sum_{i=2}^t\sum_{j=1}^d
      \abs*{\left( \frac{\eta_i}{\sqrt{\hvec{v}_i}}\right)_j -
      \left(\frac{\eta_{i-1}}{\sqrt{\hvec{v}_{i-1}}}\right)_j}}
      +
      2H^2\E{\sum_{j=1}^d\left(\frac{\eta_1}{\hvec{v}_1}\right)_j}
      \nonumber \\
    &\qquad - \E{\sum_{i=1}^t \eta_i\ip{\grad{f}(\vec{\theta}_i),\,
      \grad{f}(\vec{\theta}_i)/\sqrt{\hvec{v}_i}}}\,. \label{eq:bound_T5}
  \end{align}

\paragraph{Bound on $T_6$.}
  By the update rule $\vec{m}_i = \beta_1\vec{m}_{i-1} +
  (1-\beta_1)\vec{g}_i$, we have $\vec{m}_i = \sum_{k=1}^i
  \beta_1^{i-k}(1-\beta_1)\vec{g}_k$. From this, we have
  \begin{align}
    T_6
    &= \left(\frac{\beta_1}{1-\beta_1}\right)^2\E{\sum_{i=2}^t
      \sum_{j=1}^d
      \left(\frac{\eta_{i-1}\vec{m}_{i-1}}{\sqrt{\hvec{v}_{i-1}}}\right)_j^2}
      \nonumber \\
    &= \left(\frac{\beta_1}{1-\beta_1}\right)^2\E{\sum_{i=2}^t
      \sum_{j=1}^d\left(\sum_{k=1}^{i-1}
      \frac{\eta_{i-1}\beta_1^{i-k-1}(1-\beta_1)\vec{g}_{k}}{\sqrt{\hvec{v}_{i-1}}}\right)_j^2}
    \nonumber \\
    &=\left(\frac{\beta_1}{1-\beta_1}\right)^2\E{\sum_{i=2}^t
      \norm*{\sum_{k=1}^{i-1}
      \frac{\eta_{k}\beta_1^{i-k-1}(1-\beta_1)\vec{g}_{k}}{\sqrt{\hvec{v}_{k}}}
      + \beta_1^{i-k-1}(1-\beta_1)\vec{g}_k\left(\frac{\eta_{i-1}}{\sqrt{\hvec{v}_{i-1}}} -
      \frac{\eta_k}{\sqrt{\hvec{v}_{k}}} \right)}^2} \nonumber  \\
    &\leq 2\beta_1^2
      \E{\sum_{i=2}^t \norm*{\sum_{k=1}^{i-1}
      \frac{\eta_{k}\beta_1^{i-k-1}\vec{g}_{k}}{\sqrt{\hvec{v}_{k}}}}^2}  
      + 2\beta_1^2 \E{\sum_{i=2}^t \norm*{\sum_{k=1}^{i-1}
      \beta_1^{i-k-1}\vec{g}_k\left(\frac{\eta_{i-1}}{\sqrt{\hvec{v}_{i-1}}} -
      \frac{\eta_k}{\sqrt{\hvec{v}_{k}}} \right)}^2}\,, \label{eq:T6}
  \end{align}
  where the last inequality is due to $\norm{\vec{a} + \vec{b}}^2\leq
  2\norm{\vec{a}}^2 + 2\norm{\vec{b}}^2$.
  We bound the first term in~\eqref{eq:T6}. 
  \begin{align}
    \E{\sum_{i=2}^t \norm*{\sum_{k=1}^{i-1}
    \frac{\eta_{k}\beta_1^{i-k-1}\vec{g}_{k}}{\sqrt{\hvec{v}_{k}}}}^2}
    &= \E{\sum_{i=2}^t \sum_{j=1}^d \sum_{p=1}^{i-1}\sum_{q=1}^{i-1}
      \beta_1^{i-p-1}\left(\frac{\eta_p\vec{g}_p}{\sqrt{\hvec{v}_p}}\right)_j
      \beta_1^{i-q-1}\left(\frac{\eta_q\vec{g}_q}{\sqrt{\hvec{v}_q}}\right)_j
      } \nonumber \\
    &\leq \E{\sum_{i=2}^t \sum_{j=1}^d \sum_{p=1}^{i-1}\sum_{q=1}^{i-1}
      \beta_1^{i-p-1}\beta_1^{i-q-1}\frac{1}{2}\left(
      \left(\frac{\eta_p\vec{g}_p}{\sqrt{\hvec{v}_p}}\right)_j^2 + 
      \left(\frac{\eta_q\vec{g}_q}{\sqrt{\hvec{v}_q}}\right)_j^2 \right)
      } \nonumber  \\
    \intertext{By the symmetry of $p$ and $q$ in the summation, we have}
    &=\E{\sum_{i=2}^t \sum_{j=1}^d \sum_{p=1}^{i-1}
      \left(\frac{\eta_p\vec{g}_p}{\sqrt{\hvec{v}_p}}\right)_j^2
      \beta_1^{i-p-1}\sum_{q=1}^{i-1}\beta_1^{i-q-1}  } \nonumber \\
    &=\frac{1}{1-\beta_1}\E{\sum_{i=2}^t \sum_{j=1}^d \sum_{p=1}^{i-1}\beta_1^{i-p-1}
      \left(\frac{\eta_p\vec{g}_p}{\sqrt{\hvec{v}_p}}\right)_j^2
      }\,. \nonumber  \\
    \intertext{Changing the order of summation yields}
    &=\frac{1}{1-\beta_1}\E{
      \sum_{p=1}^{t-1}\sum_{j=1}^d\left(\frac{\eta_p\vec{g}_p}{\sqrt{\hvec{v}_p}}\right)_j^2
      \sum_{i=p+1}^t\beta_1^{i-p-1} } \nonumber  \\
    &\leq \left(\frac{1}{1-\beta_1}\right)^2 
      \E{\sum_{p=1}^{t-1}\sum_{j=1}^d\left(\frac{\eta_p\vec{g}_p}{\sqrt{\hvec{v}_p}}\right)_j^2}
      = \left(\frac{1}{1-\beta_1}\right)^2
      \E{\sum_{i=1}^{t-1}\norm*{\frac{\eta_i\vec{g}_i}{\sqrt{\hvec{v}_i}}}^2}\,. \label{eq:bound_T6_1}
  \end{align}
  For the second term in~\eqref{eq:T6}, we have
  \begin{align}
    &\E{\sum_{i=2}^t \norm*{\sum_{k=1}^{i-1}
      \beta_1^{i-k-1}\vec{g}_k\left(\frac{\eta_{i-1}}{\sqrt{\hvec{v}_{i-1}}} -
      \frac{\eta_k}{\sqrt{\hvec{v}_{k}}} \right)}^2}  \nonumber     \\
    &=\E{\sum_{i=2}^t \sum_{j=1}^d \left(\sum_{k=1}^{i-1}
      \beta_1^{i-k-1}(\vec{g}_k)_j\left(\frac{\eta_{i-1}}{\sqrt{\hvec{v}_{i-1}}} -
      \frac{\eta_k}{\sqrt{\hvec{v}_{k}}} \right)_j\right)^2}   \nonumber   \\
    &\leq H^2\E{\sum_{i=2}^t \sum_{j=1}^d \left(\sum_{k=1}^{i-1}\beta_1^{i-k-1}
      \abs*{\frac{\eta_{i-1}}{\sqrt{\hvec{v}_{i-1}}} -
      \frac{\eta_k}{\sqrt{\hvec{v}_{k}}} }_j\right)^2} \nonumber \\
    &= H^2\E{\sum_{i=1}^{t-1} \sum_{j=1}^d \left(\sum_{k=1}^{i}\beta_1^{i-k}
      \abs*{\frac{\eta_{i}}{\sqrt{\hvec{v}_{i}}} -
      \frac{\eta_k}{\sqrt{\hvec{v}_{k}}} }_j\right)^2} \nonumber \\
    &\leq H^2\E{\sum_{i=1}^{t-1} \sum_{j=1}^d\left(\sum_{k=1}^{i}\beta_1^{i-k}
      \sum_{l=k+1}^i\abs*{\frac{\eta_{l}}{\sqrt{\hvec{v}_{l}}} -
      \frac{\eta_{l-1}}{\sqrt{\hvec{v}_{l-1}}} }_j\right)^2} \nonumber \\
    &\leq
      H^2\left(\frac{1}{1-\beta_1}\right)^2\left(\frac{\beta_1}{1-\beta_1}\right)^2
      \E{\sum_{j=1}^d\sum_{l=2}^{t-1}\abs*{\frac{\eta_{l}}{\sqrt{\hvec{v}_{l}}} -
      \frac{\eta_{l-1}}{\sqrt{\hvec{v}_{l-1}}} }_j^2}\,. \label{eq:bound_T6_2}
  \end{align}
  In the above, the last inequality is due to
  Lemma~\ref{lem:bound_sqsum}.
  
  From~\eqref{eq:bound_T6_1} and~\eqref{eq:bound_T6_2}, we get
  \begin{equation} \label{eq:bound_T6}
    T_6 \leq 2\left(\frac{\beta_1}{1-\beta_1}\right)^2
    \E{\sum_{i=1}^{t-1}\norm*{\frac{\eta_i\vec{g}_i}{\sqrt{\hvec{v}_i}}}^2}
    + 2H^2\left(\frac{\beta_1}{1-\beta_1}\right)^4
    \E{\sum_{j=1}^d\sum_{l=2}^{t-1}\abs*{\frac{\eta_{l}}{\sqrt{\hvec{v}_{l}}} -
        \frac{\eta_{l-1}}{\sqrt{\hvec{v}_{l-1}}} }_j^2}
  \end{equation}
  Combining~\eqref{eq:bound_T5} together with~\eqref{eq:bound_T6}
  gives the result.
\end{proof}
\begin{lemma}
   Suppose the assumptions~\ref{as:A1}-~\ref{as:A3} hold. Then
  we have
  \begin{align*}
    \yellow{T_3}
    &=\E{\sum_{i=1}^t L\norm*{\frac{\beta_1}{1-\beta_1}
    \left(\frac{\eta_i}{\sqrt{\hvec{v}_i}} - \frac{\eta_{i-1}}{\sqrt{\hvec{v}_{i-1}}}\right) \odot
      \vec{m}_{i-1}}^2} \\
    &\leq L \left(\frac{\beta_1}{1-\beta_1}\right)^2 H^2\E{
      \sum_{i=2}^t\sum_{j=1}^d\left(\frac{\eta_i}{\sqrt{\hvec{v}_i}} -
      \frac{\eta_{i-1}}{\sqrt{\hvec{v}_{i-1}}}\right)_j^2}
  \end{align*}
\end{lemma}
\begin{proof}
  \begin{align*}
    \frac{1}{L}T_3
    &= \E{\sum_{i=2}^t \left(\frac{\beta_1}{1-\beta_1}\right)^2
      \sum_{j=1}^d\left(\frac{\eta_i}{\sqrt{\hvec{v}_i}} -
      \frac{\eta_{i-1}}{\sqrt{\hvec{v}_{i-1}}}\right)_j (\vec{m}_{i-1})_j
      }  \\
    &\leq \left(\frac{\beta_1}{1-\beta_1}\right)^2 H^2\E{
      \sum_{i=2}^t\sum_{j=1}^d\left(\frac{\eta_i}{\sqrt{\hvec{v}_i}} -
      \frac{\eta_{i-1}}{\sqrt{\hvec{v}_{i-1}}}\right)_j}\,,
  \end{align*}
  where the last inequality is due to $\norm{\vec{m}_i}<H$. This
  completes the proof.
\end{proof}

\begin{lemma} \label{lem:bound_sqsum}
  For $a_i\geq 0$, $\beta\in [0, 1)$, and $b_i =
  \beta^{i-k}\sum_{l=k+1}^ia_l$, we have
  \[
    \sum_{i=1}^t b_i^2 \leq \left(\frac{1}{1-\beta}\right)^2
    \left(\frac{\beta}{1-\beta}\right)^2 \sum_{i=2}^t a_i^2\,.
  \]
\end{lemma}
\begin{proof}
  We have
  \begin{align*}
    \sum_{i=1}^t b_i^2
    &=\sum_{i=1}^t \left(\sum_{k-1}^i
      \beta^{i-k}\sum_{l=k+1}^ia_l\right)^2 \,.\\
    \intertext{Changing the order of summation gives}
    &=\sum_{i=1}^t \left(\sum_{l=2}^i
      \sum_{k=1}^{l-1}\beta^{i-k}a_l\right)^2
      =\sum_{i=1}^t\left(\sum_{l=2}^i\beta^{i-l+1}a_l\sum_{k=1}^{l-1}\beta^{l-1-k}\right)^2\\
    &\overset{(i)}{\leq} \left(\frac{1}{1-\beta}\right)^2\sum_{i=1}^t\left(\sum_{l=2}^i\beta^{i-l+1}a_l\right)^2
    =\left(\frac{1}{1-\beta}\right)^2\sum_{i=1}^t\left(\sum_{l=2}^i\sum_{m=2}^i\beta^{i-l+1}a_l
      \beta^{i-m+1}a_m\right) \\
    &\overset{(ii)}{\leq} \left(\frac{1}{1-\beta}\right)^2
      \sum_{i=1}^t\sum_{l=2}^i\sum_{m=2}^i\beta^{i-l+1}\beta^{i-m+1}\frac{1}{2}(a_l^2
      + a_m^2)\\
    &\overset{(iii)}{=} \left(\frac{1}{1-\beta}\right)^2
      \sum_{i=1}^t\sum_{l=2}^i\sum_{m=2}^i\beta^{i-l+1}\beta^{i-m+1}a_l^2
      \leq \left(\frac{1}{1-\beta}\right)^2\frac{\beta}{1-\beta}
      \sum_{l=2}^t\sum_{i=l}^t\beta^{i-l+1}a_l^2\\
    &\leq
      \left(\frac{1}{1-\beta}\right)^2\left(\frac{\beta}{1-\beta}\right)^2
      \sum_{l=2}^ta_l^2\,,
  \end{align*}
  where (i) used $\sum_{k=1}^{l-1-k}\leq \frac{1}{1-\beta}$, (ii) is
  due to $ab \leq \frac{1}{2}(a^2 + b^2)$, (iii) is due to symmetry of
  $l$ and $m$ in the summation. This completes the proof.
\end{proof}

\begin{theorem}
  Suppose that the assumptions~\ref{as:A1}-~\ref{as:A3} are
  satisfied and let $\gamma_t = \min_{j\in [d]}\min_{\{\vec{g}_i\}_{i=1}^t}
  \frac{\eta_t}{\sqrt{\hvec{v}_t}}$. Then we have
  \begin{align*}
    \min_{t\in [T]}\, \E{\norm{\grad{f}(\vec{\theta}_t)}^2}
    &\leq
      \frac{\E{C_1\sum_{t=1}^T \norm{\frac{\eta_1\vec{g}_t}{\sqrt{\hvec{v}_t}}}^2
      + C_2 \sum_{t=2}^T \norm*{\frac{\eta_t}{\sqrt{\hvec{v}_t}} -
      \frac{\eta_{t-1}}{\sqrt{\hvec{v}_{t-1}}}}_1
      + C_3 \sum_{t=2}^T\norm*{\frac{\eta_t}{\sqrt{\hvec{v}_t}} -
      \frac{\eta_{t-1}}{\sqrt{\hvec{v}_{t-1}}}}^2      
      } + C_4}{\sum_{t=1}^T\gamma_t}\,.
  \end{align*}
\end{theorem}
\begin{proof}
  From Lemma~\ref{lem:opt_gap}, we have
  \begin{align*}
    \E{f(\vec{z}_{t+1}) - f(\vec{z}_1)}
    &\leq \sum_{i=1}^4 T_i \\
    &=-\E{\sum_{i=1}^t
      \ip*{\grad{f}(\vec{z}_i),\, \frac{\beta_1}{1-\beta_1}
      \left(\frac{\eta_i}{\sqrt{\hvec{v}_i}} - \frac{\eta_{i-1}}{\sqrt{\hvec{v}_{i-1}}}\right) \odot
      \vec{m}_{i-1}}}\,, \\
    &\qquad -\E{\sum_{i=1}^t \ip*{\grad{f}(\vec{z}_i),\, \eta_t\vec{g}_i/\sqrt{\hvec{v}_i}}}\,,\\
    & \qquad + \E{\sum_{i=1}^t L\norm*{\frac{\beta_1}{1-\beta_1}
      \left(\frac{\eta_i}{\sqrt{\hvec{v}_i}} - \frac{\eta_{i-1}}{\sqrt{\hvec{v}_{i-1}}}\right) \odot
      \vec{m}_{i-1}}^2}\,,\\
    & \qquad + \E{\sum_{i=1}^tL\norm*{\frac{\eta_i\vec{g}_i}{\sqrt{\hvec{v}_i}}}^2}\\
    &\leq
      H^2\frac{\beta_1}{1-\beta_1}\E{\sum_{i=2}^t\sum_{j=1}^d
      \abs*{\left(\frac{\eta_i}{\sqrt{\hvec{v}_i}} - \frac{\eta_{i-1}}{\sqrt{\hvec{v}_{i-1}}}\right)_j}}\\
    &\qquad  + L^2\left(\frac{\beta_1}{1-\beta_1}\right)^2
    \E{\sum_{i=1}^{t-1}\norm*{\frac{\eta_i\vec{g}_i}{\sqrt{\hvec{v}_i}}}^2}
    + L^2H^2\left(\frac{\beta_1}{1-\beta_1}\right)^4
    \E{\sum_{j=1}^d\sum_{l=2}^{t-1}\abs*{\frac{\eta_{l}}{\sqrt{\hvec{v}_{l}}} -
      \frac{\eta_{l-1}}{\sqrt{\hvec{v}_{l-1}}} }_j^2} \\
    &\qquad +2H^2\E{\sum_{i=2}^t\sum_{j=1}^d
      \abs*{\left( \frac{\eta_i}{\sqrt{\hvec{v}_i}}\right)_j -
      \left(\frac{\eta_{i-1}}{\sqrt{\hvec{v}_{i-1}}}\right)_j}}
      +
      2H^2\E{\sum_{j=1}^d\left(\frac{\eta_1}{\hvec{v}_1}\right)_j}
      \nonumber \\
    &\qquad - \E{\sum_{i=1}^t \eta_i\ip{\grad{f}(\vec{\theta}_i),\,
      \grad{f}(\vec{\theta}_i)/\sqrt{\hvec{v}_i}}}   +
      \frac{1}{2}\E{\sum_{i=2}^t\norm*{\frac{\eta_i\vec{g}_i}{\sqrt{\hvec{v}_i}}}^2}\\
    & \qquad + L \left(\frac{\beta_1}{1-\beta_1}\right)^2 H^2\E{
      \sum_{i=2}^t\sum_{j=1}^d\left(\frac{\eta_i}{\sqrt{\hvec{v}_i}} -
      \frac{\eta_{i-1}}{\sqrt{\hvec{v}_{i-1}}}\right)_j}\\
    &\qquad  +
      \E{\sum_{i=1}^tL\norm*{\frac{\eta_i\vec{g}_i}{\sqrt{\hvec{v}_i}}}^2}\,.\\
    \intertext{By merging similar terms, we get}
    &\leq\left(H^2\frac{\beta_1}{1-\beta_1} + 2H^2\right)
      \E{\sum_{i=2}^t\sum_{j=1}^d
      \abs*{\left(\frac{\eta_i}{\sqrt{\hvec{v}_i}} - \frac{\eta_{i-1}}{\sqrt{\hvec{v}_{i-1}}}\right)_j}}\\
    &\qquad +
      \left(1 + L\left(\frac{\beta_1}{1-\beta_1}\right)^2\right)LH^2\left(\frac{\beta_1}{1-\beta_1}\right)^2
      \E{\sum_{j=1}^d\sum_{i=2}^t\left(\frac{\eta_i}{\sqrt{\hvec{v}_i}} -
      \frac{\eta_{i-1}}{\sqrt{\hvec{v}_{i-1}}}\right)_j^2}\\
    &\qquad + \left(L^2\left(\frac{\beta_1}{1-\beta_1}\right)^2 +
      \frac{1}{2} + L\right)
    \E{\sum_{i=1}^{t}\norm*{\frac{\eta_i\vec{g}_i}{\sqrt{\hvec{v}_i}}}^2} \\
    &\qquad + 2H^2\E{\sum_{j=1}^d\left(\frac{\eta_1}{\hvec{v}_1}\right)_j}
      - \E{\sum_{i=1}^t \eta_i\ip{\grad{f}(\vec{\theta}_i),\,
      \grad{f}(\vec{\theta}_i)/\sqrt{\hvec{v}_i}}}\,.      \\
    \intertext{Rearranging terms gives}
    &\E{\sum_{i=1}^t \ip{\grad{f}(\vec{\theta}_i),\,
    \frac{\eta_i\grad{f}(\vec{\theta}_i)}{\sqrt{\hvec{v}_i}}}}\\
    &\leq\left(H^2\frac{\beta_1}{1-\beta_1} + 2H^2\right)
      \E{\sum_{i=2}^t\sum_{j=1}^d
      \abs*{\left(\frac{\eta_i}{\sqrt{\hvec{v}_i}} - \frac{\eta_{i-1}}{\sqrt{\hvec{v}_{i-1}}}\right)_j}}\\
    &\qquad +
      \left(1 + L\left(\frac{\beta_1}{1-\beta_1}\right)^2\right)LH^2\left(\frac{\beta_1}{1-\beta_1}\right)^2
      \E{\sum_{j=1}^d\sum_{i=2}^t\left(\frac{\eta_i}{\sqrt{\hvec{v}_i}} -
      \frac{\eta_{i-1}}{\sqrt{\hvec{v}_{i-1}}}\right)_j^2}\\
    &\qquad + \left(L^2\left(\frac{\beta_1}{1-\beta_1}\right)^2 +
      \frac{1}{2} + L\right)
      \E{\sum_{i=1}^{t}\norm*{\frac{\eta_i\vec{g}_i}{\sqrt{\hvec{v}_i}}}^2} 
      + 2H^2\E{\sum_{j=1}^d\left(\frac{\eta_1}{\hvec{v}_1}\right)_j}\\
    &\qquad + \E{f(\vec{z}_1) - f(\vec{z}_{t+1})}\\
    &\leq
      C_1\E{\sum_{i=1}^{t}\norm*{\frac{\eta_i\vec{g}_i}{\sqrt{\hvec{v}_i}}}^2}
      + C_2\E{\sum_{i=2}^t \norm*{\frac{\eta_i}{\sqrt{\hvec{v}_i}} -
      \frac{\eta_{i-1}}{\sqrt{\hvec{v}_{i-1}}}}_1}
      + C_3\E{\sum_{i=2}^t\norm*{\frac{\eta_i}{\sqrt{\hvec{v}_i}} -
      \frac{\eta_{i-1}}{\sqrt{\hvec{v}_{i-1}}}}^2} + C_4\,,
  \end{align*}
  where
  \begin{align*}
    C_1 &= \left(L^2\left(\frac{\beta_1}{1-\beta_1}\right)^2 + \frac{1}{2} + L\right)\,,\\
    C_2 &= \left(H^2\frac{\beta_1}{1-\beta_1} + 2H^2\right) \,, \\           
    C_3 &= \left(1 +
          L\left(\frac{\beta_1}{1-\beta_1}\right)^2\right)LH^2\left(\frac{\beta_1}{1-\beta_1}\right)^2\,, \\
    C_4 &= 2H^2\E{\sum_{j=1}^d\left(\frac{\eta_1}{\hvec{v}_1}\right)_j}
          + \E{f(\vec{z}_1) - f(\vec{z}^*)}\,.
  \end{align*}
  From the above, we have
  \begin{align*}
    \E{\sum_{t=1}^T \ip{\grad{f}(\vec{\theta}_t),\,
    \frac{\eta_t\grad{f}(\vec{\theta}_t)}{\sqrt{\hvec{v}_t}}}}
    &\geq \E{\sum_{t=1}^T\gamma_t\norm{\grad{f}(\vec{\theta}_t)}^2} \\
    &\geq \min_{t\in [T]}\, \E{\norm{\grad{f}(\vec{\theta}_t)}^2}\sum_{t=1}^T\gamma_t\,.
  \end{align*}
  Thus, we have
  \begin{align*}
    \min_{t\in [T]}\, \E{\norm{\grad{f}(\vec{\theta}_t)}^2}
    &\leq
      \frac{\E{C_1\sum_{t=1}^T \norm{\frac{\eta_1\vec{g}_t}{\sqrt{\hvec{v}_t}}}^2
      + C_2 \sum_{t=2}^T \norm*{\frac{\eta_t}{\sqrt{\hvec{v}_t}} -
      \frac{\eta_{t-1}}{\sqrt{\hvec{v}_{t-1}}}}_1
      + C_3 \sum_{t=2}^T\norm*{\frac{\eta_t}{\sqrt{\hvec{v}_t}} -
      \frac{\eta_{t-1}}{\sqrt{\hvec{v}_{t-1}}}}^2      
      } + C_4}{\sum_{t=1}^T\gamma_t}\,.
  \end{align*}
\end{proof}

\section{Proof of Corollary~\ref{cor:1}}
\label{proof:cor1}

\begin{proof}
We first bound non-constant terms in RHS of (\ref{ineq:thm3.1}). For the term with $C_{1}$, we have
\begin{align*}
    \E\left[ \sum_{t=1}^{T}\left\|\eta_{t}\vec{g}_{t}/\sqrt{\widehat{\vec{v}}_{t}}\right\|^{2}\right] &\leq \E\left[ \sum_{t=1}^{T}\left\|\eta_{t}\vec{g}_{t}/c_{L} \right\|^{2}\right] \quad \text{by Assumption~\ref{as:1}} \\
    &= \E\left[ \sum_{t=1}^{T}\left(\frac{1}{c_{L}\sqrt{t}}\right)^{2}   \left\|\vec{g}_{t} \right\|^{2}\right] \\
    &\leq \frac{H^{2}}{c_{L}^{2}} \sum_{t=1}^{T}\frac{1}{t} \\
    &\leq \frac{H^{2}}{c_{L}^{2}} (1+\log{T})\,,
\end{align*}
where the last inequality is due to $\sum_{t=1}^T\frac{1}{t}\leq 1 +
\log T$.

For the term with $C_{2}$, we have
\begin{align*}
     &\E\left[\sum_{t=2}^{T}\left\|
       \frac{\eta_{t}}{\sqrt{\widehat{\vec{v}}_{t}}} -
       \frac{\eta_{t-1}}{\sqrt{\widehat{\vec{v}}_{t-1}}} \right\|_{1}
       \right]  \\
     &= \E\left[\sum_{t=2}^{t_0}\left\| \frac{\eta_{t}}{\sqrt{\widehat{\vec{v}}_{t}}} - \frac{\eta_{t-1}}{\sqrt{\widehat{\vec{v}}_{t-1}}} \right\|_{1} + \sum_{t=t_0+1}^{T}\left\| \frac{\eta_{t}}{\sqrt{\widehat{\vec{v}}_{t}}} - \frac{\eta_{t-1}}{\sqrt{\widehat{\vec{v}}_{t-1}}} \right\|_{1} \right]   \\
     &=\E\left[\sum_{t=2}^{t_0}\sum_{j=1}^{d}\left| \frac{\eta_{t}}{(\sqrt{\widehat{\vec{v}}_{t}})_{j}} - \frac{\eta_{t-1}}{(\sqrt{\widehat{\vec{v}}_{t-1}})_{j}} \right| + \sum_{j=1}^{d}\sum_{t=t_0+1}^{T}\left( \frac{\eta_{t-1}}{(\sqrt{\widehat{\vec{v}}_{t-1}})_{j}}- \frac{\eta_{t}}{(\sqrt{\widehat{\vec{v}}_{t}})_{j}} \right) \right] \quad \text{by Assumption~\ref{as:2}}\\
     &\leq \E\left[\sum_{t=2}^{t_0}\frac{d}{c_{L}} + \sum_{j=1}^{d}\left( \frac{\eta_{t_0}}{(\sqrt{\widehat{\vec{v}}_{t_0})_{j}}}- \frac{\eta_{T}}{(\sqrt{\widehat{\vec{v}}_{T}})_{j}}\right) \right] \\
     &= \frac{t_0d}{c_{L}},
\end{align*}
where $\left| \eta_{t}/(\sqrt{\widehat{\vec{v}}_{t}})_{j} - \eta_{t-1}/(\sqrt{\widehat{\vec{v}}_{t-1}})_{j} \right| \leq 1/c_{L}$.

For the term with $C_{3}$, we have
\begin{align*}
    \E\left[\sum_{t=2}^{T-1}\left\| \frac{\eta_{t}}{\sqrt{\widehat{\vec{v}}_{t}}} - \frac{\eta_{t-1}}{\sqrt{\widehat{\vec{v}}_{t-1}}} \right\|^{2} \right] &= \E\left[\sum_{t=2}^{t_0}\left\| \frac{\eta_{t}}{\sqrt{\widehat{\vec{v}}_{t}}} - \frac{\eta_{t-1}}{\sqrt{\widehat{\vec{v}}_{t-1}}} \right\|^{2} + \sum_{t=t_0+1}^{T}\left\| \frac{\eta_{t}}{\sqrt{\widehat{\vec{v}}_{t}}} - \frac{\eta_{t-1}}{\sqrt{\widehat{\vec{v}}_{t-1}}} \right\|^{2}\right] \\
    &\leq \E\left[\frac{1}{c}\left(\sum_{t=2}^{t_0}\left\| \frac{\eta_{t}}{\sqrt{\widehat{\vec{v}}_{t}}} - \frac{\eta_{t-1}}{\sqrt{\widehat{\vec{v}}_{t-1}}} \right\|_{1} + \sum_{t=t_0+1}^{T}\left\| \frac{\eta_{t}}{\sqrt{\widehat{\vec{v}}_{t}}} - \frac{\eta_{t-1}}{\sqrt{\widehat{\vec{v}}_{t-1}}} \right\|_{1}\right)\right] \\
    &\leq \frac{1}{c_{L}}\left(\frac{(t_0-1)d}{c_{L}}+\frac{d}{c_{L}}\right) \\
    &=\frac{t_0d}{c_{L}^{2}}.
\end{align*}

Then we have for AdaAct,
\begin{align}
\label{ineq:upperbound}
    &\E\left[ C_{1}\sum_{t=1}^{T}\left\|\eta_{t}\vec{g}_{t}/\sqrt{\widehat{\vec{v}}_{t}}\right\|^{2} + C_{2}\sum_{t=2}^{T}\left\| \frac{\eta_{t}}{\sqrt{\widehat{\vec{v}}_{t}}} - \frac{\eta_{t-1}}{\sqrt{\widehat{\vec{v}}_{t-1}}} \right\|_{1} + C_{3}\sum_{t=2}^{T-1}\left\| \frac{\eta_{t}}{\sqrt{\widehat{\vec{v}}_{t}}} - \frac{\eta_{t-1}}{\sqrt{\widehat{\vec{v}}_{t-1}}} \right\|^{2} \right] + C_{4} \nonumber \\ 
    &\leq \frac{C_{1}H^{2}}{c_{L}^{2}} (1+\log{T})+\frac{C_{2}t_0d}{c_{L}}+\frac{C_{3}t_0d}{c_{L}^{2}}+C_{4}.
\end{align}

Now we lower bound the effective stepsizes by Assumption~\ref{as:1},
\[
\frac{\eta_{t}}{(\sqrt{\widehat{\vec{v}}_{t}})_{j}} \geq \frac{1}{c_{U}\sqrt{t}}.
\]

Thus, 
\begin{equation}
    \E\left[ \sum_{t=1}^{T}\eta_{t}\langle\grad{f(\vec{\theta}_{t})},\ \grad{f(\vec{\theta}_{t})}/\sqrt{\widehat{\vec{v}}_{t}} \rangle \right] \geq \E\left[ \sum_{t=1}^{T} \frac{1}{c_{U}\sqrt{t}} \left\|\grad{f(\vec{\theta}_{t})}\right\|^{2} \right] \geq \frac{\sqrt{T}}{c_{U}} \underset{t\in [T]}{\min}{\E\left[\left\|\grad{f(\vec{\theta}_{t})}\right\|^{2} \right]}. \label{ineq:lowerbound}
\end{equation}

Then by (\ref{ineq:thm3.1}), (\ref{ineq:upperbound}), and (\ref{ineq:lowerbound}), we have 
\[
\frac{\sqrt{T}}{c_{U}} \underset{t\in [T]}{\min}{\E\left[\left\|\grad{f(\vec{\theta}_{t})}\right\|^{2} \right]} \leq \frac{C_{1}H^{2}}{c_{L}^{2}} (1+\log{T})+\frac{C_{2}t_0d}{c_{L}}+\frac{C_{3}t_0d}{c_{L}^{2}}+C_{4}
\]
which is equivalent to
\begin{align*}
    \underset{t\in [T]}{\min}{\E\left[\left\|\grad{f(\vec{\theta}_{t})}\right\|^{2} \right]} &\leq \frac{c_{U}}{\sqrt{T}}\left( \frac{C_{1}H^{2}}{c_{L}^{2}} (1+\log{T})+\frac{C_{2}t_0d}{c_{L}}+\frac{C_{3}t_0d}{c_{L}^{2}}+C_{4} \right) \\
    &= \frac{1}{\sqrt{T}}\left(Q_{1}+Q_{2}\log{T}\right).
\end{align*}
\end{proof}

\section{Proof of Theorem~\ref{thm:stability}}
\label{proof:stability}
By definition, we have
\begin{align}
  \E\left[\Delta_{T+1}\right]
  &=\E\left[\norm{\vec{\theta}_{T+1} - \vec{\theta}_{T+1}'}_2\right] \nonumber \\
  &=\E{\norm*{\vec{\theta}_1 - \sum_{t=1}^T\frac{\eta\vec{m}_t}{\sqrt{\hvec{v}_t}}
    - \left(\vec{\theta}_1' -\sum_{t=1}^T
    \frac{\eta\vec{m}_t'}{\sqrt{\hvec{v}_t'}}\right)}} \nonumber \\
  &\leq \E{\norm{\vec{\theta}_1 - \vec{\theta}_1'}} +
    \sum_{t=1}^T\eta\E{\norm*{\frac{\vec{m}_t}{\sqrt{\hvec{v}_t}} -\frac{\vec{m}_t'}{\sqrt{\hvec{v}_t'}}}_2}\\
  &=\sum_{t=1}^T\eta\E{\norm*{\frac{\vec{m}_t}{\sqrt{\hvec{v}_t}} -\frac{\vec{m}_t'}{\sqrt{\hvec{v}_t'}}}_2}\\
  &= \sum_{t=1}^T\eta\E{\norm*{\frac{\sum_{k=1}^t
    \beta_1^{t-k}(1-\beta_1)\vec{g}_t}{\sqrt{\hvec{v}_t}}
    -\frac{\sum_{k=1}^t
    \beta_1^{t-k}(1-\beta_1)\vec{g}_t'}{\sqrt{\hvec{v}_t'}}}_2} \nonumber \\
  &= \sum_{t=1}^T\eta\E{\norm*{\frac{\sum_{k=1}^t
    \beta_1^{t-k}(1-\beta_1)\grad{f}(\vec{\theta}_k;\xi_{i_k})}{\sqrt{\hvec{v}_t}}
    -\frac{\sum_{k=1}^t \beta_1^{t-k}(1-\beta_1)\grad{f}(\vec{\theta}_k';\xi_{i_t}')}{\sqrt{\hvec{v}_t'}}}_2} \nonumber \\
  &\leq \sum_{t=1}^T\sum_{k=1}^t \eta
    \beta_1^{t-k}(1-\beta_1) \E{\norm*{\frac{\grad{f}(\vec{\theta}_k;\xi_{i_k})}{\sqrt{\hvec{v}_t}}
    - \frac{\grad{f}(\vec{\theta}_k';\xi_{i_k}')}{\sqrt{\hvec{v}_t'}}}_2
    }  \label{eq:stability_ub}
\end{align}
At iteration $k$, we have $\xi_{i_k} = \xi_{i_k}'$ with probability $1 - \frac{1}{N}$.
\begin{align*}
  & \E{\norm*{\frac{\grad{f}(\vec{\theta}_k;\xi_{i_k})}{\sqrt{\hvec{v}_t}}
    - \frac{\grad{f}(\vec{\theta}_k';\xi_{i_k}')}{\sqrt{\hvec{v}_t'}}}_2} \\
  &\leq \frac{1}{N}\E{\norm*{\frac{\grad{f}(\vec{\theta}_k;\xi_{i_k})}{\sqrt{\hvec{v}_t}}}_2}
    + \frac{1}{N}\E{\norm*{\frac{\grad{f}(\vec{\theta}_k';\xi_{i_k}')}{\sqrt{\hvec{v}_t'}}}_2}
    + \left(1-\frac{1}{N}\right) \E{\norm*{\frac{\grad{f}(\vec{\theta}_k;\xi_{i_k})}{\sqrt{\hvec{v}_t}}
    - \frac{\grad{f}(\vec{\theta}_k';\xi_{i_k})}{\sqrt{\hvec{v}_t'}}}_2}\\
  &\leq \frac{1}{N}\E{\norm*{\frac{\grad{f}(\vec{\theta}_k;\xi_{i_k})}{\sqrt{\hvec{v}_t}}}_2}
    +
    \frac{1}{N}\E{\norm*{\frac{\grad{f}(\vec{\theta}_k';\xi_{i_k}')}{\sqrt{\hvec{v}_t'}}}_2}
    + 
    \left(1-\frac{1}{N}\right) \E{\norm*{\frac{\grad{f}(\vec{\theta}_k;\xi_{i_k})}{\sqrt{\hvec{v}_t}}
    - \frac{\grad{f}(\vec{\theta}_k;\xi_{i_k})}{\sqrt{\hvec{v}_t'}}}_2} \\
  &\qquad + \left(1-\frac{1}{N}\right) \E{\norm*{\frac{\grad{f}(\vec{\theta}_k;\xi_{i_k})}{\sqrt{\hvec{v}_t'}}
    - \frac{\grad{f}(\vec{\theta}_k';\xi_{i_k})}{\sqrt{\hvec{v}_t'}}}_2} \\
\end{align*}
By plugging the above into~\eqref{eq:stability_ub}, we obtain
\begin{align*}
  \E{\Delta_{T+1}}
  &\leq \eta(1-\beta_1)\sum_{t=1}^T\sum_{k=1}^t\beta_1^{t-k}\left\{
    \frac{1}{N}\E{\norm*{\frac{\grad{f}(\vec{\theta}_k;\xi_{i_k})}{\sqrt{\hvec{v}_t}}}_2}
    +
    \frac{1}{N}\E{\norm*{\frac{\grad{f}(\vec{\theta}_k';\xi_{i_k}')}{\sqrt{\hvec{v}_t'}}}_2}\right. \\
  & \qquad \qquad + 
    \left(1-\frac{1}{N}\right) \E{\norm*{\frac{\grad{f}(\vec{\theta}_k;\xi_{i_k})}{\sqrt{\hvec{v}_t}}
    - \frac{\grad{f}(\vec{\theta}_k;\xi_{i_k})}{\sqrt{\hvec{v}_t'}}}_2} \\
  &\qquad \qquad\left. + \left(1-\frac{1}{N}\right)
    \E{\norm*{\frac{\grad{f}(\vec{\theta}_k;\xi_{i_k})}{\sqrt{\hvec{v}_t'}} 
    -
    \frac{\grad{f}(\vec{\theta}_k';\xi_{i_k})}{\sqrt{\hvec{v}_t'}}}_2}\right\}\,.\\
  &\overset{(i)}{\leq} \eta(1-\beta_1)\sum_{t=1}^T\sum_{k=1}^t\beta_1^{t-k}\left\{
    \frac{2H}{nc_L} +
    \left(1-\frac{1}{N}\right)H \E{\norm*{\frac{1}{\sqrt{\hvec{v}_t}}
    - \frac{1}{\sqrt{\hvec{v}_t'}}}_2}\right. \\
  &\qquad \qquad \left. + \left(1-\frac{1}{N}\right)
    \E{\norm*{\frac{\grad{f}(\vec{\theta}_k;\xi_{i_k})}{\sqrt{\hvec{v}_t'}} 
    -
    \frac{\grad{f}(\vec{\theta}_k';\xi_{i_k})}{\sqrt{\hvec{v}_t'}}}_2}\right\}\,.\\
  &\overset{(ii)}{\leq} \eta(1-\beta_1)\sum_{t=1}^T\sum_{k=1}^t\beta_1^{t-k}\left\{
    \frac{2H}{Nc_L} +
    \left(1-\frac{1}{N}\right)H \E{\norm*{\frac{1}{\sqrt{\hvec{v}_t}}
    - \frac{1}{\sqrt{\hvec{v}_t'}}}_2} \right.\\
  &\qquad \left. + \left(1-\frac{1}{N}\right)\frac{L}{c_L}
    \E{\norm{\vec{\theta}_k - \vec{\theta}_k'}_2}\right\}\,.\\
  &\overset{(iii)}{\leq} \frac{2\eta HT}{Nc_L} +
    \frac{\eta H(N-1)}{N} \sum_{t=1}^T\underbracket{\E{\norm*{\frac{1}{\sqrt{\hvec{v}_t}}
    - \frac{1}{\sqrt{\hvec{v}_t'}}}_2}}_{A}
    + \frac{\eta(1-\beta_1)L}{c_L}\sum_{t=1}^T\underbracket{\E{\sum_{k=1}^t\beta_1^{t-k}\Delta_k}}_{B}\,,
\end{align*}
where (i) is due to the upper bound on gradient and lower bounded on
activation variance, (ii) is due to the Lipschitz continuity of
gradient, and (iii) is obtained by applying
$\sum_{k=1}^{t}\beta^{t-k}\leq \frac{1}{1-\beta}$.

\newpage
\section{Bounded Activations from AdaAct}
\label{apdx:adaact_actvar_bound}

\begin{figure}[ht]
\centering
\includegraphics[width=1\columnwidth]{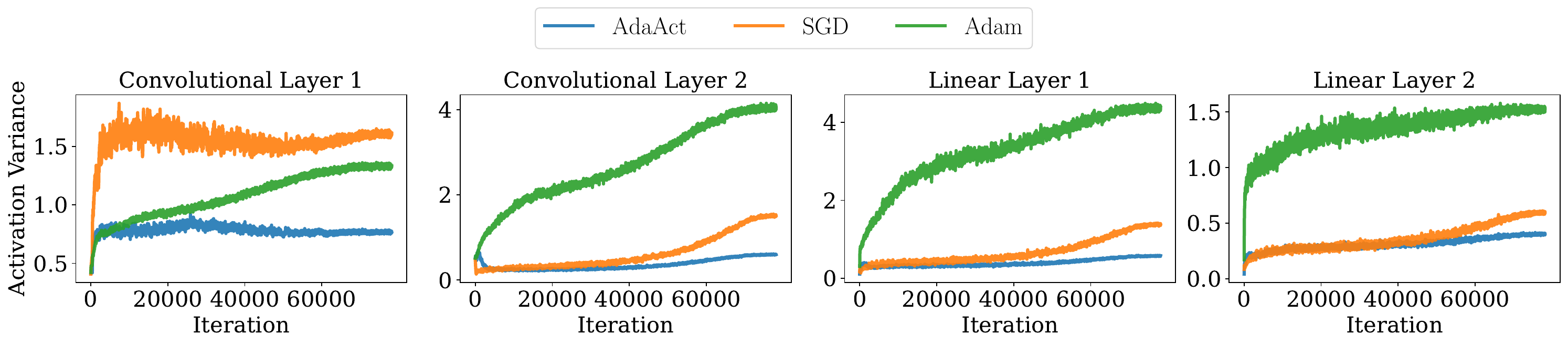}
\caption{Activation variance resulted from training LeNet-5 on CIFAR10}
\label{fig:cifar10_lenet5_actvar_all}
\end{figure}

We trained LeNet-5 on CIFAR10 for 200 epochs to observe the trend of activation variance over iterations. Figure~\ref{fig:cifar10_lenet5_actvar_all} presents the average of the activation variances across all hidden layers in the architecture. We observe that the activations from layers are bounded.

\section{Hyperparameters of Opimizers for Training CIFAR datasets}
\label{apdx:hyperpar}

\begin{table}[ht]
    \caption{Hyperparameter values used in CIFAR datasets training}
    \label{tab:cifar_hyperpar}
    \centering
    \begin{tabular}{c|cccccccccc}
        \toprule
         & $\eta$ & Mom. & $\beta_{1}$ & $\beta_{2}$ & $\beta_{3}$ & $\lambda$ & $\epsilon$ & $\delta$ & $T_{cov}$ & $T_{inv}$ \\ 
         \midrule
         SGD & 0.1 & 0.9 & . & . & . & $5\times 10^{-4}$ & . & . & . & .\\
         \adaact & 0.1 & . & 0.9 & 0.999 & . & $2\times 10^{-3}$ & $1\times 10^{-8}$ & . & . & .\\
         Adam & 0.001 & . & 0.9 & 0.999 & . & $5\times 10^{-4}$ & $1\times 10^{-8}$ & . & . & .\\
         AdamW & 0.001 & . & 0.9 & 0.999 & . & $1\times 10^{-2}$ & $1\times 10^{-8}$ & . & . & .\\
         Adan & 0.01 & . & 0.98 & 0.92 & 0.99 & $1\times 10^{-2}$ & $1\times 10^{-8}$ & . & . & .\\
         FOOF & 0.05 & 0.9 & . & 0.95 & . & $5\times 10^{-4}$ & . & 1 & 5 & 50\\
         KFAC & 0.05 & 0.9 & . & 0.9 & . & $5\times 10^{-4}$ & . & 1, 10 & 5 & 50\\
         \bottomrule
    \end{tabular}
\end{table}

$\delta$ denotes the damping factor, $T_{cov}$ is the update period for the covariance matrix of activations or pre-activation gradients, and $T_{inv}$ represents the update period for the inverse of the preconditioning matrix used in FOOF and KFAC. For those two optimizers, $\beta_{2}$ indicates the exponential moving average coefficient for the preconditioner.


\end{document}